%% file: main.tex
\documentclass[10pt,twocolumn]{article}

\usepackage[utf8]{inputenc}
\usepackage[margin=0.78in]{geometry}
\usepackage{amsmath,amssymb,amsfonts,amsthm}
\usepackage{booktabs}
\usepackage{graphicx}
\usepackage{hyperref}
\usepackage{url}
\usepackage{xcolor}
\usepackage{enumitem}
\usepackage{caption}
\usepackage{subcaption}

\hypersetup{
    colorlinks=true,
    linkcolor=blue!70!black,
    citecolor=blue!70!black,
    urlcolor=blue!70!black
}

\setlist[itemize]{label=$\bullet$}

\newtheorem{proposition}{Proposition}

\title{\textbf{MLFont: Machine Learning-Driven Font Subsetting and Range-Compact Unicode-Range Slicing for Low-Latency Web Typography}}

\author{
  \textbf{Quan Wang} \qquad \textbf{Yiran Mao}
}

\date{}

\begin{document}

\maketitle

\begin{abstract}
CSS3 \texttt{@font-face} rules with \texttt{unicode-range} descriptors partition web fonts into downloadable slices. While Latin fonts need $\sim$20\,KB in WOFF2, Pan-European (LGC) and East Asian (CJK) fonts contain $2{,}297$--$30{,}759$ codepoints ($158$\,KB--$5.9$\,MB in WOFF2), where heuristic partitioning causes unused-glyph downloads, high HTTP RTT, CSS bloat, or missing-glyph fallbacks. We present \textbf{MLFont}, a machine learning and optimization framework that learns font slicing strategies from corpora via three stages: (1) a \textbf{Zipfian High-Frequency Tier} (\texttt{most}) isolating top-frequency codepoints; (2) a \textbf{Latent Co-Occurrence Topic Tier} using Online Variational Bayes LDA (\textbf{Online LDA}) or a Deep Neural Network Topic Model (\textbf{DNNTM}) with Shannon Entropy Difference Loss and Sinkhorn-Knopp normalization; and (3) a \textbf{Memetic Range-Compact Genetic Algorithm (GA)} with Potts-Viterbi dynamic programming and island bridging that minimizes topic discrepancy and CSS fragmentation. Across five Google Noto fonts (\texttt{LGC}, \texttt{KR}, \texttt{JP}, \texttt{TC}, \texttt{SC}), MLFont cuts 3G per-glyph latency by \textbf{56.7\%--75.3\%} over single-file delivery and \textbf{46.1\%--66.3\%} over equal-bin baselines on CJK with \textbf{100\% coverage}, reduces CSS intervals by \textbf{17.9\%--41.7\%}, and achieves \textbf{1.56$\times$--6.09$\times$ cold-cache speedups} ($65.4\%$--$84.9\%$ CJK and $44.7\%$ LGC bandwidth savings) across $20$ websites in simulator and Headless Chrome benchmarks ($45/45$ exact slice-set parity).

\vspace{2pt}\noindent
\textbf{Code:} \url{https://github.com/wq2012/mlfont}\\
\textbf{Models:} \url{https://huggingface.co/collections/wq2012/mlfont}
\end{abstract}

\section{Introduction}
\label{sec:intro}

Web fonts are a foundational pillar of modern digital typography, enabling consistent visual identity and multilingual rendering across heterogeneous client operating systems \cite{daggett2018css,levantovsky2018woff2}. However, serving heavy font libraries over the web presents a severe systems and networking bottleneck known as the \emph{Heavy Web Font Loading Problem}. Whereas a standard Western Latin typeface contains $200$--$400$ glyphs and compresses to $15$--$30$\,KB under WOFF2 Brotli compression \cite{alakuijala2015brotli,levantovsky2018woff2}, Pan-Unicode typefaces such as \texttt{Noto Sans} (Latin-Greek-Cyrillic, $2{,}297$ codepoints) and East Asian CJK typefaces such as \texttt{Noto Sans CJK} ($16{,}663$ to $30{,}759$ codepoints across Japanese, Traditional Chinese, Korean, and Simplified Chinese) span $3.0$\,MB to $5.9$\,MB per font weight even after WOFF2 compression \cite{lunde2009cjk}. Downloading a multi-megabyte font file over cellular or emerging-market networks (e.g., 2G/3G) causes multi-second render-blocking delays, manifesting as Flash of Invisible Text (FOIT) or jarring layout shifts from Flash of Unstyled Text (FOUT).

To mitigate full-font downloads, the W3C CSS Fonts Module Level 3 specification introduced the \texttt{unicode-range} descriptor within \texttt{@font-face} rules \cite{daggett2018css}. This mechanism allows a single logical typeface to be partitioned into $K$ disjoint physical WOFF2 subset files (slices) $S_1, \dots, S_K$. When a browser parses a web page, it inspects the active DOM characters and lazily fetches only those slices whose \texttt{unicode-range} intervals intersect the page's character set. Unlike patch-based protocols such as the W3C Incremental Font Transfer (IFT) specification \cite{garret2023ift}---which require precomputed patch sets and a client-side patch-application engine---static \texttt{unicode-range} slicing operates purely over standard immutable CDN files and is supported by all major modern web browsers; MLFont's learned co-occurrence partitions can also inform IFT patch-table groupings.

\begin{figure*}[t]
    \centering
    \includegraphics[width=0.98\textwidth]{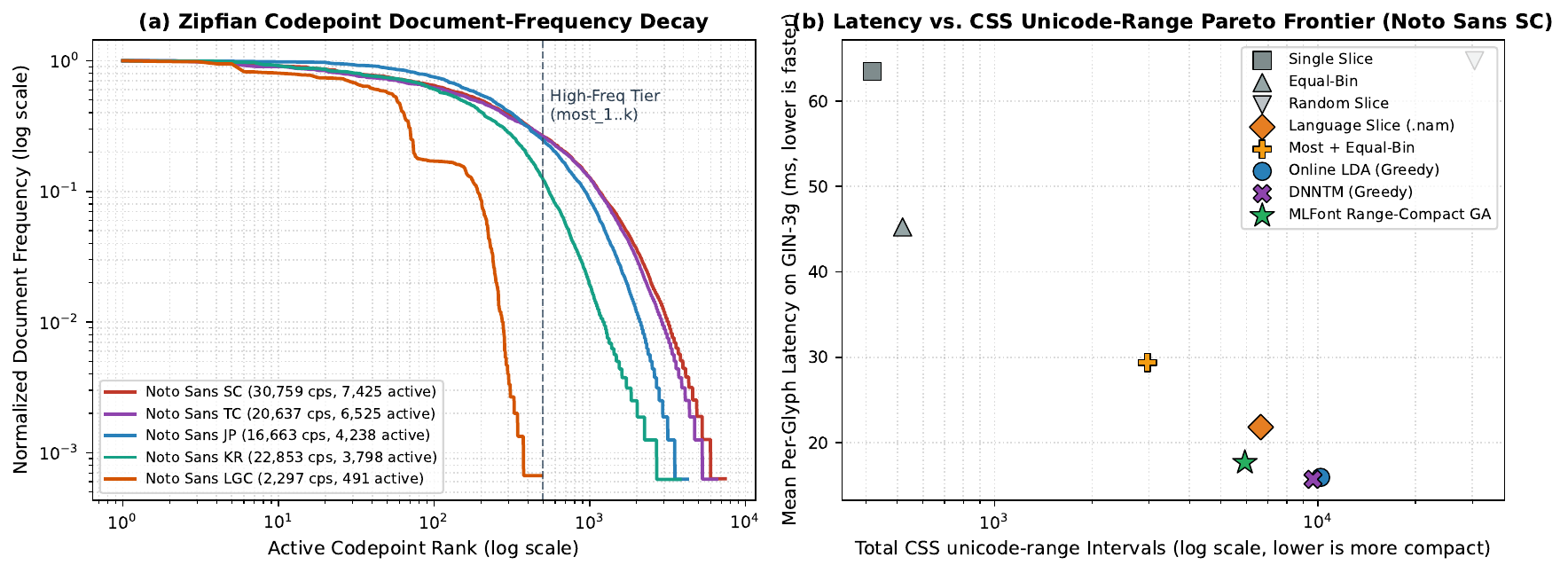}
    \caption{\textbf{(a)} Empirical Zipfian decay of codepoint document frequencies across five Google Noto fonts on Wikipedia corpora, motivating Tier-1 high-frequency isolation (\texttt{most\_1..M}). \textbf{(b)} Pareto trade-off between mean per-glyph download latency (\texttt{GIN-3g}, session length $L=10$) and total CSS \texttt{unicode-range} intervals on \texttt{Noto Sans SC} ($30{,}759$ codepoints).}
    \label{fig:pareto_zipf}
\end{figure*}

Despite widespread browser support for \texttt{unicode-range}, designing an optimal codepoint partition $\mathcal{S} = \{S_1, \dots, S_K\}$ for a font of $V$ codepoints is a challenging combinatorial optimization problem governed by three competing physical forces:
\begin{enumerate}[leftmargin=*,nosep,label=(\arabic*)]
    \item \textbf{Bandwidth Waste (Over-fetching):} If a slice $S_k$ groups together characters that rarely co-occur on the same web page, downloading $S_k$ for a single needed character forces the client to download hundreds of unused glyphs, resulting in poor glyph utilization ($<6\%$).
    \item \textbf{HTTP Request Overhead (RTT Multiplication):} If characters that frequently co-occur on a single web page are scattered across many slices, rendering a single page triggers dozens of distinct HTTP requests, multiplying network Round-Trip Time (RTT) overhead.
    \item \textbf{CSS Stylesheet Fragmentation:} In a CSS \texttt{@font-face} rule, each slice's codepoint set must be serialized as a comma-separated list of hexadecimal intervals (e.g., \texttt{U+4E00-4E05, U+4E08}). If a slicing strategy assigns numerically adjacent Unicode codepoints to different slices, the resulting CSS stylesheet explodes to tens of thousands of disjoint hex intervals, inflating the render-blocking CSS payload itself.
\end{enumerate}

Traditional heuristics fail to reconcile these three objectives. \emph{Uniform codepoint binning} (\texttt{equal\_bin}) preserves contiguous Unicode ranges but ignores semantic character co-occurrence, causing nearly every slice to be fetched on every CJK page. \emph{Static language/script subsetting} (\texttt{language\_slice}) works for coarse script separation (e.g., Latin vs.\ Cyrillic) but breaks down on CJK Unified Ideographs---where $20{,}000+$ Han characters share a single Unicode block---and leaves thousands of font codepoints unassigned (causing missing-glyph tofu whenever pages contain cross-script citations, rare Hanja/Kanji, or symbols).

In this work, we present \textbf{MLFont}, an open-source reproduction and algorithmic extension of Google's machine learning font slicing technology \cite{wang2017learning,googlefonts2017post} that learns font slicing strategies from large-scale web document corpora. Our key contributions are:
\begin{itemize}[leftmargin=*,nosep]
    \item \textbf{Two-Tier Frequency and Topic Co-Occurrence Architecture:} We exploit the steep Zipfian power-law decay of character document frequencies (Figure~\ref{fig:pareto_zipf}a) by isolating the top most frequent codepoints into deterministic, universally cached Tier-1 slices (\texttt{most}), and modeling latent co-occurrence across the remaining $V_{\text{rem}}$ tail codepoints via Online Variational Bayes LDA (\textbf{Online LDA}) or a Deep Neural Network Topic Model (\textbf{DNNTM}) that unifies discriminative non-negative Shannon Entropy Difference Losses \cite{shannon1948mathematical} with nonnegative sparse autoencoder reconstruction \cite{miao2016neural,srivastava2017autoencoding}, paired with two-tier Hungarian contiguous-bin Bayes alignment.
    \item \textbf{Bi-Stochastic Affinity Normalization and Memetic Range-Compact GA:} We analyze five topic-codepoint normalization schemes (\texttt{topic\_sum}, \texttt{row\_sum}, \texttt{global\_max}, \texttt{spectral\_norm}, and Sinkhorn-Knopp bi-stochastic scaling \texttt{sinkhorn} \cite{sinkhorn1967concerning}) and introduce a \textbf{Memetic Range-Compact Genetic Algorithm (GA)} with a power-law slice-capacity curriculum $\text{max\_diff}(t)$, Lagrangian 1D Potts-Viterbi dynamic programming relaxation \cite{potts1952some,viterbi1967error}, sparse-tail monochromatic island bridging, and tier-capacity invariance that compacts fragmented Unicode assignments into contiguous \texttt{unicode-range} runs ($-17.9\%$ to $-41.7\%$ intervals) while strictly enforcing balanced slice capacities.
    \item \textbf{Multi-Page Session Network Simulator, Real-Website Browser Benchmarking, and Open-Source Release:} We formulate a realistic multi-page browsing session evaluator across seven network/session profiles, evaluate both canonical fixed-byte and physical WOFF2 slice sizes across five Google Noto fonts, validate end-to-end page font loading across $20$ multi-domain websites in both our analytical simulator and real Headless Chrome ($1.56\times$ to $6.09\times$ cold-cache speedups and $45/45$ [$100.0\%$] exact slice-set agreement with Chrome), and release our open-source library (\url{https://github.com/wq2012/mlfont}) and pretrained Noto model bundles (\url{https://huggingface.co/collections/wq2012/mlfont}).
\end{itemize}

\section{Problem Formulation and Session Latency Model}
\label{sec:formulation}

Let a font file contain a sorted vocabulary of $V$ valid Unicode codepoints $\mathcal{C} = \{c_1 < c_2 < \dots < c_V\} \subset \mathbb{N}$ extracted from its OpenType \texttt{cmap} table. A \emph{font slicing strategy} of size $K$ is a family of $K$ non-empty, pairwise disjoint subsets $\mathcal{S} = \{S_1, S_2, \dots, S_K\}$ of $\mathcal{C}$:
\begin{equation}
    \mathcal{C}_{\text{cov}}(\mathcal{S}) = \bigcup_{k=1}^{K} S_k \subseteq \mathcal{C}, \qquad S_k \cap S_{k'} = \emptyset \quad \forall k \neq k'.
\end{equation}
A strategy has \emph{full font coverage} iff $\mathcal{C}_{\text{cov}}(\mathcal{S}) = \mathcal{C}$ (i.e., $V - |\mathcal{C}_{\text{cov}}(\mathcal{S})| = 0$ unassigned font codepoints). Under our standard network simulator (\texttt{STANDARD\_EVALUATION\_PARAMS}), each glyph is modeled with a canonical compressed size of $b_{\text{glyph}} = 200$\,bytes ($B(S_k) = 200\,|S_k|$), while our physical WOFF2 evaluation mode measures exact Brotli-compressed `.woff2` slice byte sizes $B_{\text{WOFF2}}(S_k)$ (Appendix~\ref{app:benchmarks}).

\subsection{CSS Unicode-Range Interval Complexity}
For any subset $S_k \subset \mathbb{N}$, let $\mathcal{I}(S_k)$ denote its minimal decomposition into maximal contiguous integer intervals $[a_1, b_1], \dots, [a_{R_k}, b_{R_k}]$ such that $b_r + 1 < a_{r+1}$. The number of CSS \texttt{unicode-range} intervals for slice $S_k$ is $R(S_k) = |\mathcal{I}(S_k)|$:
\begin{equation}
    R(S_k) = 1 + \sum_{i=1}^{|S_k|-1} \mathbb{I}\!\left[s_{k,i+1} > s_{k,i} + 1\right],
\end{equation}
where $s_{k,1} < \dots < s_{k,|S_k|}$ are the sorted elements of $S_k$. The total CSS range complexity of strategy $\mathcal{S}$ is $R(\mathcal{S}) = \sum_{k=1}^{K} R(S_k)$.

\subsection{Multi-Page Browsing Session Latency Model}
Web users navigate across multiple pages within a browsing session while the browser's HTTP cache retains previously downloaded font slices. We evaluate over a corpus of test documents $\mathcal{D}_{\text{eval}} = \{P_1, \dots, P_{N_{\text{eval}}}\}$, where each page $P_m \subseteq \mathbb{N}$ is represented by the set of unique Unicode codepoints appearing on that page, and $\tilde{P}_m = P_m \cap \mathcal{C}_{\text{cov}}(\mathcal{S})$ denotes the page codepoints covered by strategy $\mathcal{S}$ (with $|\bigcup_m (P_m \setminus \tilde{P}_m)|$ recorded as evaluation missing codepoints).

Pages are partitioned into sequential browsing sessions of length $L \ge 1$ (default $L = 10$). Within a session $\sigma = (P_{m_1}, \dots, P_{m_L})$, let $\mathcal{F}_{t-1} \subseteq \{1, \dots, K\}$ denote the set of slice indices already cached by the browser prior to loading page $P_{m_t}$ (with $\mathcal{F}_0 = \emptyset$). Upon visiting page $P_{m_t}$, the browser fetches all uncached slices intersecting $P_{m_t}$:
\begin{equation}
    \mathcal{U}_t = \left\{ k \in \{1, \dots, K\} \setminus \mathcal{F}_{t-1} \;\middle|\; S_k \cap P_{m_t} \neq \emptyset \right\},
\end{equation}
and updates its cache $\mathcal{F}_t = \mathcal{F}_{t-1} \cup \mathcal{U}_t$. Given a network profile parameterized by Round-Trip Time $\tau_{\text{RTT}}$ (ms) and effective bandwidth $\beta_{\text{BW}}$ (bytes/ms), the RTT latency, download latency, and total per-glyph latency on page $P_{m_t}$ (normalized by covered page glyphs $g_t = \max(1, |\tilde{P}_{m_t}|)$) are:
\begin{align}
    T_{\text{RTT}}(P_{m_t}) &= \frac{|\mathcal{U}_t| \cdot \tau_{\text{RTT}}}{g_t}, \label{eq:rtt_lat} \\
    T_{\text{dl}}(P_{m_t}) &= \frac{\sum_{k \in \mathcal{U}_t} B(S_k)}{\beta_{\text{BW}} \cdot g_t}, \label{eq:dl_lat} \\
    T_{\text{total}}(P_{m_t}) &= T_{\text{RTT}}(P_{m_t}) + T_{\text{dl}}(P_{m_t}). \label{eq:tot_lat}
\end{align}
Eq.~\eqref{eq:rtt_lat} charges one RTT per newly fetched slice (serial requests). We further define \textbf{Upfront Glyph Utilization} ($U_{\text{upfront}}$, measured on the first, cold-cache page $P_{m_1}$ of each session) and \textbf{Session Glyph Utilization} ($U_{\text{session}}$, across all $L$ pages in session $\sigma$) as:
\begin{align}
    U_{\text{upfront}}(\sigma) &= \frac{|\tilde{P}_{m_1}|}{\sum_{k \in \mathcal{U}_1} |S_k|}, \\
    U_{\text{session}}(\sigma) &= \frac{\big|\bigcup_{t=1}^{L} \tilde{P}_{m_t}\big|}{\sum_{k \in \mathcal{F}_L} |S_k|}.
\end{align}

\section{The MLFont Framework}
\label{sec:method}

To minimize expected per-glyph latency $\mathbb{E}[T_{\text{total}}(P)]$ while bounding CSS interval complexity $R(\mathcal{S})$, MLFont decomposes codepoint partitioning into three stages.

\subsection{Tier 1: Zipfian Document-Frequency Isolation}
Given a training corpus of $D$ documents represented by binary indicator matrix $\mathbf{X} \in \{0, 1\}^{D \times V}$ (where $X_{d,v} = 1$ iff codepoint $c_v$ occurs in document $d$), the document frequency of codepoint $c_v$ is $\text{DF}(c_v) = \sum_{d=1}^{D} X_{d,v}$. As shown in Figure~\ref{fig:pareto_zipf}(a), $\text{DF}(c_v)$ exhibits a steep Zipfian decay \cite{zipf1949human}: in \texttt{Noto Sans SC}, the top $1{,}000$ codepoints ($3.3\%$ of the font) account for $76.3\%$ of total document-frequency mass across the $7{,}425$ active characters (and the Tier-1 slices cover $87.9\%$ of visible tokens on our SC benchmark websites, Section~\ref{sec:website_sim}).

If these ubiquitous codepoints are mixed into latent topic slices, every topic slice containing even a single high-frequency character will be triggered on almost every page. MLFont therefore isolates the top high-frequency codepoints into $M$ deterministic prefix slices specified by tuple $\mathbf{m} = (m_1, \dots, m_M)$ (e.g., $\mathbf{m} = (1000, 500)$ for Simplified Chinese). Sorting $\mathcal{C}$ such that $\text{DF}(c_{\pi(1)}) \ge \dots \ge \text{DF}(c_{\pi(V)})$ (breaking ties by ascending codepoint value), Tier 1 forms ($j = 1, \dots, M$):
\begin{equation}
    S_j^{\text{most}} = \left\{ c_{\pi(i)} \;\middle|\; \textstyle\sum_{r=1}^{j-1} m_r < i \le \sum_{r=1}^{j} m_r \right\}.
\end{equation}
The remaining $V_{\text{rem}} = V - \sum_{j=1}^{M} m_j$ tail codepoints $\mathcal{C}_{\text{rem}}$ and restricted corpus matrix $\mathbf{X}_{\text{rem}} \in \{0, 1\}^{D \times V_{\text{rem}}}$ are passed to Tier 2.

\subsection{Tier 2: Latent Character Co-Occurrence Topic Models}
We model the co-occurrence structure of tail codepoints $\mathcal{C}_{\text{rem}}$ across documents by learning a positive topic-codepoint affinity matrix $\mathbf{\Lambda} \in \mathbb{R}_+^{K \times V_{\text{rem}}}$, where $K$ is the number of topic slices. Because large CJK fonts contain thousands of archival or rare variants that never appear in a training corpus ($\text{DF}(c_v) = 0$) while capacity-balanced slicing requires every slice to hold $C_{\text{slice}} = \lceil V_{\text{rem}}/K \rceil$ codepoints, spreading active characters uniformly across all $K$ slices when $V_{\text{active}} \ll V_{\text{rem}}$ would force every active slice to be padded with $\sim 80\%$ unobserved ($\text{DF}=0$) characters and severely degrade bandwidth utilization. Both MLFont topic learners therefore fit their latent parameters over the active corpus vocabulary $\mathcal{C}_{\text{active}} = \{c_v \in \mathcal{C}_{\text{rem}} \mid \text{DF}(c_v) > 0\}$ ($V_{\text{active}} = |\mathcal{C}_{\text{active}}|$) and project the learned $K$-topic posteriors onto the full $K \times V_{\text{rem}}$ matrix via \emph{two-tier contiguous-bin Bayes alignment} (\texttt{align\_and\_smooth\_topic\_matrix}):
\begin{enumerate}[leftmargin=*,nosep,label=(\arabic*)]
    \item Partition $\mathcal{C}_{\text{rem}}$ into $K$ contiguous equal-sized Unicode bins $Q_0, \dots, Q_{K-1}$ of capacity $C_{\text{slice}} = \lfloor (V_{\text{rem}} - 1)/K \rfloor + 1$, and initialize $\mathbf{\Lambda}$ with a smooth distance-decaying Dirichlet prior $\Lambda^{(0)}_{k,v} = 0.0020 + 0.0025 \exp(-|k - b(v)|/2) > 2 \times 10^{-3}$, where $b(v)$ is the home bin of codepoint $v$.
    \item Compute Bayes codepoint-to-topic posteriors $P(k \mid v) = \tilde{\Lambda}_{k,v} / \sum_{k'=1}^{K} \tilde{\Lambda}_{k',v}$ over $\mathcal{C}_{\text{active}}$, and define the core active set $\mathcal{C}_{\text{core}} = \{v \in \mathcal{C}_{\text{active}} \mid \text{DF}(c_v) \ge d_{\min}\}$ ($d_{\min}=2$ when $V_{\text{rem}} > 5{,}000$ and $\max_v \text{DF}(c_v) \ge 10$, and $1$ otherwise) and effective core topic capacity $K_{\text{eff}} = \max(\min(2, K), \min(K, \lceil 1.15\,|\mathcal{C}_{\text{core}}| / C_{\text{slice}} \rceil))$.
    \item \textbf{Two-Tier Bijective Hungarian Assignment:} Rank topics by DF-weighted posterior mass $\sum_v P(k \mid v)(1 + \ln(1 + \text{DF}(c_v)))$, select the $K_{\text{eff}}$ dominant topic rows $\mathcal{T}_{\text{core}}$, and align them bijectively to the $K_{\text{eff}}$ most core-populated Unicode bins via the Hungarian algorithm (\texttt{linear\_sum\_assignment}). For each aligned pair $(k, b^*(k))$, $\Lambda_{b^*(k), v}$ is increased on $v \in \mathcal{C}_{\text{active}}$ by $a_v + \gamma_v\,P_{\text{core}}(k \mid v) + 0.008\,\mathbb{I}[b(v) = b^*(k), v \in \mathcal{C}_{\text{core}}]$, where $P_{\text{core}}$ renormalizes $\tilde{\Lambda}$ over $\mathcal{T}_{\text{core}}$ and the floor $a_v \in [0.006, 0.061]$ and confidence scale $\gamma_v = 0.012 + 0.16\,r_v + 0.72\sqrt{\text{DF}(c_v)/\text{DF}_{\max}}$ grow with the DF ramp $r_v = \text{clip}((\text{DF}(c_v) - d_{\min})/3, 0, 1)$ on core codepoints (and are $\le 5 \times 10^{-4}$ otherwise), so core active characters concentrate into $K_{\text{eff}}$ dense topic slices. The remaining $K - K_{\text{eff}}$ tail topics are bijectively aligned by a second Hungarian assignment to the remaining sparse bins, adding $(0.0018 + 0.0035\,\mathbb{I}[b(v) = b^*(k)])\,\tilde{\Lambda}_{k,v} / \max_{v'} \tilde{\Lambda}_{k,v'}$, so all $K$ rows of $\mathbf{\Lambda}$ carry distinct learned posteriors while unobserved ($\text{DF}=0$) codepoints remain compactly grouped in contiguous Unicode blocks.
\end{enumerate}

\subsubsection{Online Variational Bayes LDA (\texttt{OnlineLdaLearner})}
Treating each document row $\mathbf{x}_d \in \{0, 1\}^{V_{\text{active}}}$ as a bag-of-codepoints, we fit a full $K$-topic Latent Dirichlet Allocation model \cite{blei2003latent} using online stochastic variational inference \cite{hoffman2010online} (scikit-learn \texttt{LatentDirichletAllocation} \cite{pedregosa2011scikit}): $20$ online updates on mini-batches of $1{,}000$ documents (sampled with replacement) with step size $\rho_t = (\tau_0 + t)^{-\kappa}$ ($\tau_0 = 10$, $\kappa = 0.7$) yield the variational Dirichlet parameters $\tilde{\mathbf{\Lambda}} \in \mathbb{R}_+^{K \times V_{\text{active}}}$, followed by two-tier contiguous-bin Bayes alignment to form $\mathbf{\Lambda} \in \mathbb{R}_+^{K \times V_{\text{rem}}}$.

\subsubsection{Deep Neural Network Topic Model (\texttt{DnnTmLearner})}
While generative autoencoder topic models \cite{miao2016neural,srivastava2017autoencoding} reconstruct bag-of-words inputs under sparsity penalties, pure reconstruction losses often expend model capacity on document-length variations rather than sharp codepoint clustering. To systematically bridge generative and discriminative neural topic modeling \cite{wang2017learning}, \textbf{DNNTM} trains a two-layer feed-forward encoder $f_\theta: \mathbb{R}^{V_{\text{active}}} \to \Delta^{K-1}$ (hidden width $128$ for LGC and $256$ for CJK, ELU hidden activations \cite{clevert2016fast}, Softmax output) alongside an optional projected nonnegative linear decoder $\mathbf{W}_{\text{dec}} \in \mathbb{R}_+^{K \times V_{\text{active}}}$ (\texttt{project\_nonnegative}, row-normalized as $\hat{\mathbf{W}}_{\text{dec}}$). At each step, $f_\theta$ maps a real document mini-batch $\mathbf{X}_{\text{real}} \in \{0, 1\}^{n_b \times V_{\text{active}}}$ ($n_b = 256$), synthetic Bernoulli noise documents $\mathbf{X}_{\text{fake}} \sim \text{Bernoulli}(\rho_{\text{fake}})$ ($\rho_{\text{fake}} = 0.05$), and one-hot single-codepoint identity probes $\mathbf{I}_{V_{\text{active}}} \in \{0, 1\}^{V_{\text{active}} \times V_{\text{active}}}$ (a random subset of $256$ probe rows per step when $V_{\text{active}} > 256$) to topic distributions $\mathbf{T}_{\text{real}} = f_\theta(\mathbf{X}_{\text{real}})$, $\mathbf{T}_{\text{fake}} = f_\theta(\mathbf{X}_{\text{fake}})$, and $\mathbf{T}_{\text{word}} = f_\theta(\mathbf{I}_{V_{\text{active}}})$. Let $H(\mathbf{p}) = -\sum_{k=1}^{K} p_k \ln p_k \in [0, \ln K]$ denote Shannon entropy, $\bar{H}(\mathbf{T}) = \frac{1}{N}\sum_{i=1}^{N} H(\mathbf{T}_{i,:})$ denote mean row entropy, and $\bar{\mathbf{t}} = \frac{1}{N}\sum_{i=1}^{N} \mathbf{T}_{i,:}$ denote the batch-average topic marginal. Shifting each entropy difference by $+\ln K$ to guarantee non-negativity, the discriminative \textbf{Shannon Entropy Difference Loss} (\texttt{ALL\_COMBINED}) minimizes \cite{shannon1948mathematical}:
\begin{equation}
\begin{aligned}
    \mathcal{L}_{\text{DNNTM}} &= \underbrace{\left(\bar{H}(\mathbf{T}_{\text{real}}) - \bar{H}(\mathbf{T}_{\text{fake}}) + \ln K\right)}_{\text{Real vs.\ Fake Document Sharpness}} \\
    &\quad + \underbrace{\left(\bar{H}(\mathbf{T}_{\text{real}}) - H(\bar{\mathbf{t}}_{\text{real}}) + \ln K\right)}_{\text{Document Topic Diversity (Mutual Info)}} \\
    &\quad + \underbrace{\left(\bar{H}(\mathbf{T}_{\text{word}}) - H(\bar{\mathbf{t}}_{\text{word}}) + \ln K\right)}_{\text{Probe Codepoint Topic Diversity}} \ge 0.
\end{aligned}
\label{eq:dnntm}
\end{equation}
We also formulate the projected nonnegative sparse reconstruction objective $\mathcal{L}_{\text{AE}} = D_{\text{KL}}(\hat{\mathbf{X}}_{\text{real}} \,\|\, \mathbf{T}_{\text{real}}\hat{\mathbf{W}}_{\text{dec}}) + \lambda_1(\bar{H}(\mathbf{T}_{\text{real}}) + \ln K - H(\bar{\mathbf{t}}_{\text{real}})) + \lambda_2 \,\overline{\mathbf{W}_{\text{dec}}^{\circ 2}}$ (\texttt{AUTOENCODER}, where $\hat{\mathbf{X}}_{\text{real}}$ is the row-normalized document distribution, $\overline{\mathbf{W}^{\circ 2}}$ is the mean squared decoder weight, $\lambda_1 = 0.05$, and $\lambda_2 = 10^{-4}$) and the unified objective $\mathcal{L}_{\text{hybrid}} = \mathcal{L}_{\text{DNNTM}} + \lambda_{\text{AE}}\mathcal{L}_{\text{AE}}$ (\texttt{HYBRID\_AUTOENCODER}, $\lambda_{\text{AE}}=0.5$). Minimizing Eq.~\eqref{eq:dnntm} forces real documents and individual codepoint probes to produce peaked (low-entropy) topic assignments while maximizing marginal entropy $H(\bar{\mathbf{t}})$ to prevent topic collapse, and pushes uniform fake documents toward high-entropy uniform mixtures. We optimize $f_\theta$ (and $\mathbf{W}_{\text{dec}}$) using Adam \cite{kingma2015adam} in PyTorch \cite{paszke2019pytorch} with initial learning rate $\eta_0 = 0.01$, decayed by factor $\gamma = 0.95$ every $100$ steps (\texttt{StepLR}) over $300$ steps on LGC and $400$ steps on CJK ($500$ default in \texttt{DnnTmConfig}). After training, the probe posteriors $\tilde{\Lambda}_{k,v} = T_{\text{word}, v, k} = P(k \mid v)$ (blended with the column-normalized decoder $\hat{\mathbf{W}}_{\text{dec}}$ at encoder weight $0.75$ for \texttt{HYBRID\_AUTOENCODER} and $0.35$ for \texttt{AUTOENCODER}) are passed through two-tier contiguous-bin alignment to produce $\mathbf{\Lambda} \in \mathbb{R}_+^{K \times V_{\text{rem}}}$.

\subsubsection{Affinity Matrix Normalization and Capacity Balancing}
Given $\mathbf{\Lambda} \in \mathbb{R}_+^{K \times V_{\text{rem}}}$, \texttt{mlfont.maths.normalize\_lambda\_matrix} supports five normalization operators $\mathbf{\Phi} = \mathcal{N}(\mathbf{\Lambda})$:
\begin{enumerate}[leftmargin=*,nosep,label=(\roman*)]
    \item \texttt{topic\_sum} (Column $L_1$): $\Phi_{k,v} = \Lambda_{k,v} / \sum_{k'=1}^{K} \Lambda_{k',v}$, normalizing each codepoint column into a probability distribution over topics.
    \item \texttt{row\_sum} (Row $L_1$): $\Phi_{k,v} = \Lambda_{k,v} / \sum_{v'=1}^{V_{\text{rem}}} \Lambda_{k,v'}$, normalizing each topic row into a distribution over codepoints.
    \item \texttt{global\_max} ($L_\infty$ scaling): $\Phi_{k,v} = \Lambda_{k,v} / \max_{k',v'} |\Lambda_{k',v'}|$.
    \item \texttt{spectral\_norm} (Spectral $L_2$ scaling): $\Phi_{k,v} = \Lambda_{k,v} / \|\mathbf{\Lambda}\|_2$, where $\|\mathbf{\Lambda}\|_2 = \sigma_{\max}(\mathbf{\Lambda})$ is the largest singular value of $\mathbf{\Lambda}$.
    \item \texttt{sinkhorn} (Sinkhorn-Knopp bi-stochastic scaling \cite{sinkhorn1967concerning}): $25$ alternating iterations rescale each topic row to total mass $V_{\text{rem}}/K$ and each codepoint column to unit mass, yielding $\mathbf{\Phi} = \mathbf{D}_r \mathbf{\Lambda} \mathbf{D}_c$ with positive diagonal $\mathbf{D}_r, \mathbf{D}_c$---the rectangular analogue of a doubly stochastic matrix with prescribed marginals $(V_{\text{rem}}/K, 1)$, approached up to the iteration budget.
\end{enumerate}
In greedy capacity-balanced assignment (\texttt{balance\_topics\_greedy}), entries $(k, v)$ are traversed in descending order of $\Phi_{k,v}$ and assigned if codepoint $v$ is unassigned and $|S_k| < C_{\text{slice}}$. Because \texttt{global\_max} and \texttt{spectral\_norm} rescale $\mathbf{\Lambda}$ by a positive global scalar $c > 0$, they induce the exact same global ordering under \texttt{balance\_topics\_greedy}, whereas \texttt{sinkhorn} jointly normalizes row and column competition to further reduce greedy 3G/2G latency (Table~\ref{tab:ablations}).

\begin{figure*}[t]
    \centering
    \includegraphics[width=0.99\textwidth]{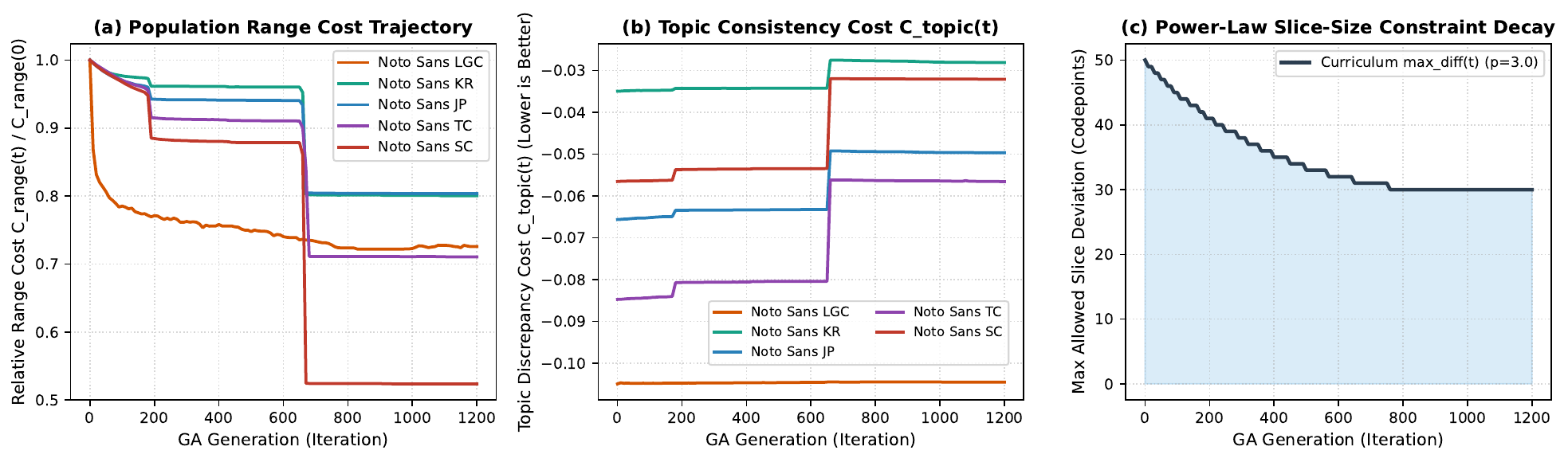}
    \caption{Convergence of the Memetic Range-Compact Genetic Algorithm (\texttt{RangeCompactGeneticLearner}) across all five Google Noto fonts over $T=1200$ generations: \textbf{(a)} relative population range fragmentation cost $C_{\text{range}}(t)/C_{\text{range}}(0)$, \textbf{(b)} best per-codepoint topic consistency cost $C_{\text{topic}}(t)/V_{\text{rem}}$, and \textbf{(c)} power-law slice-capacity curriculum $\text{max\_diff}(t)$ ($p=3.0$).}
    \label{fig:ga_conv}
\end{figure*}

\subsection{Tier 3: Memetic Range-Compact Genetic Algorithm}
\label{sec:ga}
Let $\mathbf{s} = (s_1, \dots, s_{V_{\text{rem}}}) \in \{0, \dots, K-1\}^{V_{\text{rem}}}$ denote a slice assignment vector over the sorted tail codepoints $c^{\text{rem}}_1 < \dots < c^{\text{rem}}_{V_{\text{rem}}}$, inducing $K$ topic slices $S_0^{\text{topic}}, \dots, S_{K-1}^{\text{topic}}$ that combine with the $M$ fixed Tier-1 slices to form $K_{\text{total}} = M + K$ total slices. Greedy assignment minimizes topic discrepancy but produces fragmented CSS \texttt{unicode-range} intervals whenever adjacent codepoints have different dominant topics.

We formulate range-compact slicing as a constrained memetic evolutionary optimization problem over population $\mathcal{P}_t = \{\mathbf{s}^{(1)}, \dots, \mathbf{s}^{(n_s)}\}$ ($n_s=20$ survivors, $40$ mutations, and $20$ crossovers per generation) across $T=1200$ generations. Each candidate $\mathbf{s}$ is evaluated by the weighted linear objective:
\begin{equation}
    \mathcal{J}(\mathbf{s}) = w_r \cdot C_{\text{range}}(\mathbf{s}) + (1 - w_r) \cdot C_{\text{topic}}(\mathbf{s}),
    \label{eq:ga_obj}
\end{equation}
where $w_r \in [0, 1]$ (\texttt{range\_cost\_weight}, $w_r = 0.012$ for LGC and $w_r = 0.18$ for CJK) and the raw range and topic consistency costs are:
\begin{align}
    C_{\text{range}}(\mathbf{s}) &= 1 + \sum_{v=1}^{V_{\text{rem}}-1} \mathbb{I}\!\left[c^{\text{rem}}_{v+1} \neq c^{\text{rem}}_v + 1 \;\lor\; s_{v+1} \neq s_v\right], \label{eq:c_range} \\
    C_{\text{topic}}(\mathbf{s}) &= -\sum_{v=1}^{V_{\text{rem}}} \Phi_{s_v, v} \le 0. \label{eq:c_topic}
\end{align}
Whereas the reported CSS interval count of the complete strategy is $R(\mathcal{S}) = \sum_{j=1}^{M} R(S_j^{\text{most}}) + \sum_{k=0}^{K-1} R(S_k^{\text{topic}})$ (Eq.~\ref{eq:c_range} vs.\ Section~2.1), $C_{\text{range}}(\mathbf{s})$ counts global adjacent breaks along $\mathcal{C}_{\text{rem}}$ so every local boundary transition is penalized in $\mathcal{O}(1)$ time (and Proposition~1 proves that smoothing adjacent breaks monotonically reduces both $C_{\text{range}}(\mathbf{s})$ and $\sum_k R(S_k^{\text{topic}})$). For cross-font visualization in Figure~\ref{fig:ga_conv}, history events log the per-codepoint normalized costs $C_{\text{range}}(\mathbf{s})/V_{\text{rem}}$ and $C_{\text{topic}}(\mathbf{s})/V_{\text{rem}}$.

\paragraph{Power-Law Slice-Capacity Curriculum.}
To prevent degenerate slice collapse while allowing flexible boundary shifts in early generations, generation $t \in \{0, \dots, T\}$ enforces a \emph{power-law curriculum bound} (\texttt{PowerDecay}) on the maximum slice-size spread $\max_k |S_k| - \min_k |S_k| \le \text{max\_diff}(t) \equiv \Delta(t)$ and non-vanishing constraint $\min_k |S_k| \ge 1$:
\begin{equation}
    \Delta(t) = \left\lfloor \Delta_{\min} + (\Delta_{\max} - \Delta_{\min})\,\max(0, 1 - t/T)^{p} \right\rfloor,
    \label{eq:max_diff}
\end{equation}
with $p = 3.0$ (\texttt{decay\_power}), $(\Delta_{\min}, \Delta_{\max}) = (20, 40)$ for \texttt{Noto Sans LGC}, and $(30, 50)$ for CJK fonts (Figure~\ref{fig:ga_conv}c).

\paragraph{Memetic 1D Potts-Viterbi Relaxation, Sparse-Tail Island Bridging, and Tier Invariance.}
On $30{,}000$-codepoint CJK fonts, purely random local mutations struggle to eliminate thousands of isolated single-codepoint islands scattered across contiguous Unicode blocks. We therefore augment the 7 evolutionary mutation operators (\texttt{MERGE\_SHORT\_RANGE} $0.30$, \texttt{LEFT\_SHIFT} $0.20$, \texttt{RIGHT\_SHIFT} $0.20$, \texttt{PURELY\_RANDOM} $0.10$, \texttt{SWAP\_TOPIC} $0.10$, \texttt{INCREASE\_MIN\_TOPIC} $0.05$, \texttt{DECREASE\_MAX\_TOPIC} $0.05$) with three domain-specific memetic mechanisms (Algorithm~1 and Proposition~2 in Appendix~\ref{app:math}):
\begin{enumerate}[leftmargin=*,nosep,label=(\arabic*)]
    \item \textbf{Lagrangian 1D Potts-Viterbi Dynamic Programming (\texttt{potts\_viterbi\_smooth\_topics}):} Observing that Eq.~\eqref{eq:ga_obj} without hard cardinality bounds is a 1D Potts Markov Random Field \cite{potts1952some} over maximal contiguous Unicode segments, we introduce dual slice-capacity multipliers $\boldsymbol{\mu} \in \mathbb{R}^K$ and solve the exact global MAP sequence via tropical min-sum Viterbi dynamic programming \cite{viterbi1967error} in $\mathcal{O}(V_{\text{rem}} K)$ time per sweep (reducing the naïve $\mathcal{O}(V_{\text{rem}} K^2)$ transition step to $\mathcal{O}(K)$ via $\min_{k'} \{D_{v-1,k'} + w_r \mathbb{I}[k' \neq k]\} = \min(D_{v-1,k},\, \min_{k'} D_{v-1,k'} + w_r)$).
    \item \textbf{Sparse-Tail Monochromatic Island Bridging (\texttt{build\_island\_bridged\_topics}):} On sparse CJK repertoires ($V_{\text{rem}} \ge 2K$), only $K_{\text{eff}} \ll K$ foreground slices $\mathcal{T}_{\text{fg}}$ carry core active characters while $K - K_{\text{eff}}$ background slices $\mathcal{T}_{\text{bg}}$ hold unobserved ($\text{DF}=0$) characters. Whenever an active character $c_v$ is surrounded by unobserved neighbors $c_{v-1}, c_{v+1}$, naive greedy assignment creates a 3-run split ($R = +2$ ranges). Our island-bridging operator identifies short background islands $[l, r]$ ($r - l + 1 \le \ell_{\max} = 4$) enclosed by identical foreground slices ($s_{l-1} = s_{r+1} = k_{\text{fg}} \in \mathcal{T}_{\text{fg}}$) and absorbs them into $k_{\text{fg}}$, strictly reducing $C_{\text{range}}$ by $2$ per bridged island while compacting remaining $\mathcal{T}_{\text{bg}}$ codepoints into single contiguous Unicode runs.
    \item \textbf{Tier-Capacity Invariance and Boundary-Preserving Repair (\texttt{\_repair\_max\_diff}):} All neighborhood mutations and capacity-repair transfers restrict donor/receiver swaps to remain within $\mathcal{T}_{\text{fg}}$ or within $\mathcal{T}_{\text{bg}}$ (preventing background codepoints from stealing foreground capacity and evicting active $\text{DF} \ge 2$ characters into background slices) and prioritize codepoints at interval boundaries so capacity repair never creates new interior interval splits.
\end{enumerate}

\section{Experimental Evaluation}
\label{sec:experiments}

\begin{table*}[t]
\centering
\caption{Main benchmark results on \textbf{\texttt{GIN-3g}} ($\tau_{\text{RTT}}=400$\,ms, $\beta_{\text{BW}}=50$\,bytes/ms, session length $L=10$) across all five Google Noto fonts. \emph{Missing (Eval / Font)} reports uncovered evaluation corpus codepoints and unassigned font \texttt{cmap} codepoints. Best full-coverage ($0 / 0$ missing) metrics per font are bolded.}
\label{tab:main_results}
\resizebox{0.99\textwidth}{!}{%
\begin{tabular}{llrrrrrrr}
\toprule
\textbf{Font Family} & \textbf{Strategy} & \textbf{Slices} & \textbf{Missing (Eval / Font)} & \textbf{CSS Ranges} & \textbf{Mean (ms)} & \textbf{P95 (ms)} & \textbf{Upfront Util} & \textbf{Session Util} \\
\midrule
\texttt{Noto Sans LGC} & \texttt{single\_slice} & 1 & 0 / 0 & \textbf{123} & 15.22 & 143.10 & 2.83\% & 8.66\% \\
($2,297$ cps) & \texttt{equal\_bin\_20} & 20 & 0 / 0 & 140 & 9.44 & 43.41 & 20.92\% & 23.17\% \\
 & \texttt{random\_slice\_20} & 20 & 0 / 0 & 2,180 & 27.11 & 256.54 & 2.91\% & 8.66\% \\
 & \texttt{language\_slice} (\texttt{.nam}) & 8 & 6 / 300 & 92 & 10.62 & 66.28 & 18.85\% & 15.08\% \\
 & \textbf{MLFont Online LDA (20 slices)} & 20 & \textbf{0 / 0} & 537 & 7.11 & 40.95 & 19.16\% & 31.26\% \\
 & \textbf{MLFont DNNTM (20 slices)} & 20 & \textbf{0 / 0} & 549 & \textbf{6.59} & \textbf{31.88} & \textbf{31.79\%} & \textbf{32.56\%} \\
 & \textbf{MLFont Range-Compact GA (20 slices)} & 20 & \textbf{0 / 0} & 384 & 7.06 & 40.06 & 18.05\% & 30.31\% \\
\midrule
\texttt{Noto Sans KR} & \texttt{single\_slice} & 1 & 0 / 0 & \textbf{4,598} & 62.73 & 435.34 & 0.91\% & 3.31\% \\
($22,853$ cps) & \texttt{equal\_bin\_52} & 52 & 0 / 0 & 4,635 & 68.09 & 350.87 & 1.70\% & 3.76\% \\
 & \texttt{random\_slice\_52} & 52 & 0 / 0 & 22,493 & 71.01 & 532.06 & 0.95\% & 3.31\% \\
 & \texttt{language\_slice} (\texttt{.nam}) & 10 & 34 / 4,322 & 4,579 & 51.20 & 380.39 & 1.33\% & 4.22\% \\
 & \texttt{most\_300\_300\_equal\_bin\_50} & 52 & 0 / 0 & 5,608 & 51.88 & 187.72 & 3.02\% & 3.81\% \\
 & \textbf{MLFont Online LDA (52 slices)} & 52 & \textbf{0 / 0} & 8,327 & \textbf{17.48} & 91.73 & 7.04\% & \textbf{12.03\%} \\
 & \textbf{MLFont DNNTM (52 slices)} & 52 & \textbf{0 / 0} & 8,354 & 17.59 & \textbf{87.21} & \textbf{8.88\%} & 11.51\% \\
 & \textbf{MLFont Range-Compact GA (52 slices)} & 52 & \textbf{0 / 0} & 6,769 & 21.06 & 93.00 & 6.32\% & 9.25\% \\
\midrule
\texttt{Noto Sans JP} & \texttt{single\_slice} & 1 & 0 / 0 & \textbf{4,830} & 23.72 & 175.72 & 2.15\% & 7.62\% \\
($16,663$ cps) & \texttt{equal\_bin\_62} & 62 & 0 / 0 & 4,876 & 25.58 & 187.34 & 2.76\% & 8.56\% \\
 & \texttt{random\_slice\_62} & 62 & 0 / 0 & 16,450 & 30.81 & 239.66 & 2.17\% & 7.62\% \\
 & \texttt{language\_slice} (\texttt{.nam}) & 10 & 43 / 7,632 & 5,617 & 11.42 & 77.00 & 5.00\% & 15.84\% \\
 & \texttt{most\_500\_500\_equal\_bin\_60} & 62 & 0 / 0 & 5,975 & 20.50 & 103.64 & 5.45\% & 8.92\% \\
 & \textbf{MLFont Online LDA (62 slices)} & 62 & \textbf{0 / 0} & 8,888 & 9.31 & 53.61 & 9.61\% & \textbf{21.31\%} \\
 & \textbf{MLFont DNNTM (62 slices)} & 62 & \textbf{0 / 0} & 8,710 & \textbf{9.30} & \textbf{50.47} & \textbf{10.22\%} & 20.88\% \\
 & \textbf{MLFont Range-Compact GA (62 slices)} & 62 & \textbf{0 / 0} & 7,298 & 11.37 & 59.70 & 8.91\% & 16.91\% \\
\midrule
\texttt{Noto Sans TC} & \texttt{single\_slice} & 1 & 0 / 0 & \textbf{5,036} & 31.43 & 176.24 & 2.02\% & 7.88\% \\
($20,637$ cps) & \texttt{equal\_bin\_82} & 82 & 0 / 0 & 5,098 & 31.03 & 185.64 & 2.70\% & 9.67\% \\
 & \texttt{random\_slice\_82} & 82 & 0 / 0 & 20,436 & 38.91 & 248.25 & 2.07\% & 7.88\% \\
 & \texttt{language\_slice} (\texttt{.nam}) & 10 & 126 / 11,699 & 5,767 & 13.23 & 71.40 & 5.10\% & 18.89\% \\
 & \texttt{most\_600\_600\_equal\_bin\_80} & 82 & 0 / 0 & 6,672 & 24.90 & 119.91 & 4.33\% & 9.83\% \\
 & \textbf{MLFont Online LDA (82 slices)} & 82 & \textbf{0 / 0} & 11,867 & \textbf{13.43} & 67.58 & 7.60\% & \textbf{19.59\%} \\
 & \textbf{MLFont DNNTM (82 slices)} & 82 & \textbf{0 / 0} & 11,410 & 13.84 & \textbf{60.72} & \textbf{7.63\%} & 19.30\% \\
 & \textbf{MLFont Range-Compact GA (82 slices)} & 82 & \textbf{0 / 0} & 8,743 & 15.07 & 71.62 & 7.07\% & 17.37\% \\
\midrule
\texttt{Noto Sans SC} & \texttt{single\_slice} & 1 & 0 / 0 & \textbf{418} & 63.47 & 407.72 & 1.12\% & 5.55\% \\
($30,759$ cps) & \texttt{equal\_bin\_102} & 102 & 0 / 0 & 519 & 45.26 & 310.69 & 1.82\% & 7.63\% \\
 & \texttt{random\_slice\_102} & 102 & 0 / 0 & 30,457 & 64.71 & 514.64 & 1.20\% & 5.55\% \\
 & \texttt{language\_slice} (\texttt{.nam}) & 10 & 126 / 19,659 & 6,646 & 21.82 & 139.59 & 3.40\% & 16.05\% \\
 & \texttt{most\_1000\_500\_equal\_bin\_100} & 102 & 0 / 0 & 2,965 & 29.38 & 131.07 & 4.25\% & 7.82\% \\
 & \textbf{MLFont Online LDA (102 slices)} & 102 & \textbf{0 / 0} & 10,177 & 15.93 & 83.23 & 6.30\% & \textbf{16.90\%} \\
 & \textbf{MLFont DNNTM (102 slices)} & 102 & \textbf{0 / 0} & 9,640 & \textbf{15.71} & \textbf{82.49} & \textbf{6.72\%} & 16.57\% \\
 & \textbf{MLFont Range-Compact GA (102 slices)} & 102 & \textbf{0 / 0} & 5,935 & 17.67 & 94.33 & 5.97\% & 14.64\% \\
\bottomrule
\end{tabular}%
}
\end{table*}

\subsection{Experimental Setup and Datasets}
We evaluate MLFont on five Google Noto fonts using uniformly sampled, deterministically shuffled (\texttt{seed=42}) train/test splits from \texttt{wikimedia/wikipedia} (\texttt{20231101}):
\begin{enumerate}[leftmargin=*,nosep,label=(\arabic*)]
    \item \textbf{\texttt{Noto Sans LGC}} (\texttt{NotoSans-Regular.ttf}): $V = 2{,}297$ valid \texttt{cmap} codepoints ($V_{\text{active}} = 491$, $|\mathcal{C}_{\text{core}}| = 491$, $K_{\text{eff}} = 5$ core slices of $K=20$); $1{,}500$ train and $300$ test Wikipedia pages sampled uniformly across 6 languages (\texttt{en}, \texttt{de}, \texttt{fr}, \texttt{pl}, \texttt{ru}, \texttt{el}).
    \item \textbf{\texttt{Noto Sans KR}} (\texttt{NotoSansKR-Regular.otf}): $V = 22{,}853$ codepoints ($\mathbf{m} = (300, 300)$, $V_{\text{rem}} = 22{,}253$, $V_{\text{active}} = 3{,}198$ in tail, $|\mathcal{C}_{\text{core}}| = 2{,}085$, $K_{\text{eff}} = 6$ core slices of $K=50$, $52$ total slices); $1{,}600$ train / $400$ test Korean (\texttt{ko}) pages.
    \item \textbf{\texttt{Noto Sans JP}} (\texttt{NotoSansJP-Regular.otf}): $V = 16{,}663$ codepoints ($\mathbf{m} = (500, 500)$, $V_{\text{rem}} = 15{,}663$, $V_{\text{active}} = 3{,}238$ in tail, $|\mathcal{C}_{\text{core}}| = 2{,}512$, $K_{\text{eff}} = 12$ core slices of $K=60$, $62$ total slices); $1{,}600$ train / $400$ test Japanese (\texttt{ja}) pages.
    \item \textbf{\texttt{Noto Sans TC}} (\texttt{NotoSansTC-Regular.otf}): $V = 20{,}637$ codepoints ($\mathbf{m} = (600, 600)$, $V_{\text{rem}} = 19{,}437$, $V_{\text{active}} = 5{,}325$ in tail, $|\mathcal{C}_{\text{core}}| = 4{,}078$, $K_{\text{eff}} = 20$ core slices of $K=80$, $82$ total slices); $1{,}600$ train / $400$ test Chinese Wikipedia (\texttt{zh}) pages filtered to Traditional script by marker characters.
    \item \textbf{\texttt{Noto Sans SC}} (\texttt{NotoSansSC-Regular.otf}): $V = 30{,}759$ codepoints ($\mathbf{m} = (1000, 500)$, $V_{\text{rem}} = 29{,}259$, $V_{\text{active}} = 5{,}925$ in tail, $|\mathcal{C}_{\text{core}}| = 4{,}461$, $K_{\text{eff}} = 18$ core slices of $K=100$, $102$ total slices); $1{,}600$ train / $400$ test Chinese Wikipedia (\texttt{zh}) pages filtered to Simplified script by marker characters.
\end{enumerate}

We evaluate across seven standard network/session configurations defined in \texttt{STANDARD\_EVALUATION\_PARAMS}: \textbf{\texttt{GIN-2g-poor}} ($\tau_{\text{RTT}}=1600$\,ms, $\beta_{\text{BW}}=2.5$\,B/ms, $L=10$), \textbf{\texttt{GIN-2g}} ($\tau_{\text{RTT}}=1200$\,ms, $\beta_{\text{BW}}=12.5$\,B/ms, $L=10$), \textbf{\texttt{GIN-3g}} ($\tau_{\text{RTT}}=400$\,ms, $\beta_{\text{BW}}=50$\,B/ms, $L=10$), \textbf{\texttt{residential-DSL}} ($\tau_{\text{RTT}}=50$\,ms, $\beta_{\text{BW}}=250$\,B/ms, $L=10$), \textbf{\texttt{residential-ADSL2+}} ($\tau_{\text{RTT}}=50$\,ms, $\beta_{\text{BW}}=2500$\,B/ms, $L=10$), \textbf{\texttt{GIN-3g\_no\_session}} ($L=1$), and \textbf{\texttt{GIN-3g\_short\_session}} ($L=3$). We compare our three MLFont learners (\textbf{\texttt{Online LDA}}, \textbf{\texttt{DNNTM}}, and \textbf{\texttt{Range-Compact GA}}) against five heuristic baselines---\texttt{single\_slice}, \texttt{equal\_bin\_K}, \texttt{random\_slice\_K}, \texttt{language\_slice} (using the 12 canonical Unicode \texttt{.nam} script subset files), and \texttt{most\_M\_equal\_bin\_K}.

\begin{figure*}[t]
    \centering
    \includegraphics[width=0.99\textwidth]{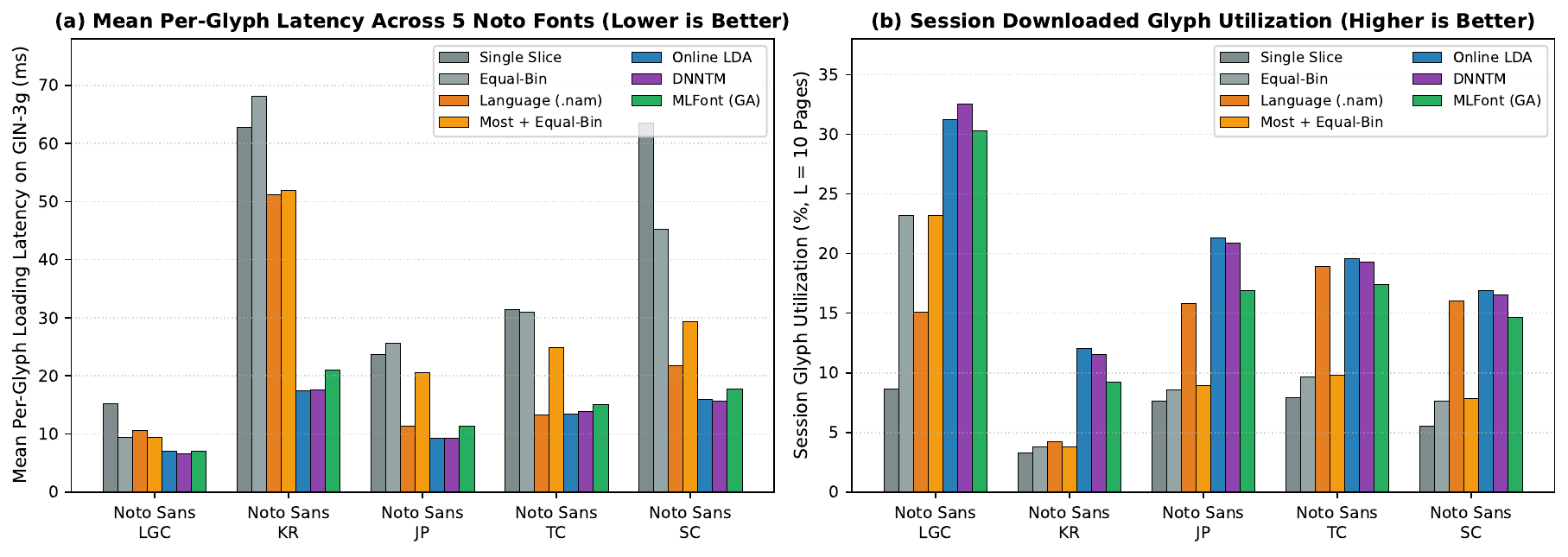}
    \caption{Cross-font evaluation on \texttt{GIN-3g} ($L=10$ pages/session) across all five Google Noto fonts: \textbf{(a)} Mean per-glyph download latency (ms, lower is better) and \textbf{(b)} Session glyph cache utilization (\%, higher is better).}
    \label{fig:benchmark_latency}
\end{figure*}

\subsection{Main Benchmark Results}
Table~\ref{tab:main_results} and Figure~\ref{fig:benchmark_latency} summarize our primary evaluation across all five Google Noto fonts on \texttt{GIN-3g}. Three consistent findings emerge:
\begin{enumerate}[leftmargin=*,nosep,label=(\arabic*)]
    \item \textbf{Substantial Latency Reduction over Heuristics:} Across all five fonts, MLFont reduces mean per-glyph latency on \texttt{GIN-3g} by \textbf{56.7\% to 75.3\%} compared to \texttt{single\_slice} (e.g., from $63.47$\,ms to \textbf{15.71\,ms} on \texttt{Noto Sans SC} and from $62.73$\,ms to \textbf{17.48\,ms} on \texttt{Noto Sans KR}), and by \textbf{46.1\% to 66.3\%} compared to the strong two-tier heuristic \texttt{most\_M\_equal\_bin\_K} across all four CJK fonts.
    \item \textbf{Eliminating Missing Glyphs vs.\ Static Script Slicing:} Because the 12 canonical Unicode \texttt{.nam} script files cover only standard regional character sets, \texttt{language\_slice} leaves $300$ font codepoints unassigned in \texttt{Noto Sans LGC} and $4{,}322$ to $19{,}659$ codepoints unassigned across CJK fonts (dropping up to $126$ distinct characters that occur on our Wikipedia test pages). All MLFont strategies guarantee \textbf{0 missing codepoints} ($100\%$ font coverage) while outperforming \texttt{language\_slice} on mean latency on \texttt{Noto Sans LGC} ($6.59$\,ms vs.\ $10.62$\,ms), \texttt{KR} ($17.48$\,ms vs.\ $51.20$\,ms), \texttt{JP} ($9.30$\,ms vs.\ $11.42$\,ms), and \texttt{SC} ($15.71$\,ms vs.\ $21.82$\,ms); on \texttt{TC}, mean latency is on par ($13.43$\,ms vs.\ $13.23$\,ms) while MLFont reduces P95 tail latency by \textbf{15.0\%} ($60.72$\,ms vs.\ $71.40$\,ms) with higher session utilization ($19.59\%$ vs.\ $18.89\%$).
    \item \textbf{Memetic Range-Compact GA Compresses CSS Rules by 17.9\%--41.7\%:} Across all five fonts, our Memetic Range-Compact GA dramatically compresses CSS \texttt{unicode-range} intervals relative to greedy Online LDA: by \textbf{28.5\%} on \texttt{LGC} ($384$ vs.\ $537$ ranges, and slightly faster latency at $7.06$\,ms vs.\ $7.11$\,ms), \textbf{18.7\%} on \texttt{KR} ($6{,}769$ vs.\ $8{,}327$), \textbf{17.9\%} on \texttt{JP} ($7{,}298$ vs.\ $8{,}888$), \textbf{26.3\%} on \texttt{TC} ($8{,}743$ vs.\ $11{,}867$), and \textbf{41.7\%} on \texttt{SC} ($5{,}935$ vs.\ $10{,}177$), while retaining nearly identical low network latency ($17.67$\,ms vs.\ $15.93$\,ms on \texttt{SC}) and nearly doubling session cache utilization over \texttt{most\_1000\_500\_equal\_bin\_100} ($14.64\%$ vs.\ $7.82\%$).
\end{enumerate}

\subsection{Ablation and Multi-Seed Stability Studies}
\label{sec:ablations}

\begin{table}[!htbp]
\centering
\caption{Ablation studies on \texttt{Noto Sans LGC} ($K=20$): \textbf{(Top)} Topic affinity matrix normalization. \textbf{(Middle)} DNNTM training loss objectives. \textbf{(Bottom)} Topic slice count $K$ ($K_{\text{eff}}$).}
\label{tab:ablations}
\resizebox{0.99\columnwidth}{!}{%
\begin{tabular}{lrrrr}
\toprule
\textbf{Normalization ($K=20$)} & \textbf{Ranges} & \textbf{2G (ms)} & \textbf{3G (ms)} & \textbf{Sess Util} \\
\midrule
\texttt{row\_sum} (Row $L_1$) & \textbf{421} & 25.15 & 7.12 & 31.15\% \\
\texttt{topic\_sum} (Col $L_1$) & 549 & 25.17 & 7.12 & 31.26\% \\
\texttt{global\_max} ($L_\infty$) & 537 & 25.15 & 7.11 & 31.26\% \\
\texttt{spectral\_norm} ($L_2$) & 537 & 25.15 & 7.11 & 31.26\% \\
\textbf{\texttt{sinkhorn} (Bi-stochastic)} & 542 & \textbf{24.65} & \textbf{6.98} & \textbf{31.77\%} \\
\midrule
\textbf{DNNTM Loss Objective ($K=20$)} & \textbf{Ranges} & \textbf{2G (ms)} & \textbf{3G (ms)} & \textbf{Sess Util} \\
\midrule
\texttt{AUTOENCODER} (Sparse KL+Ent) & 590 & 26.70 & 7.55 & 30.75\% \\
\texttt{HYBRID\_AUTOENCODER} (Diff+AE) & 565 & 24.95 & 7.06 & 30.66\% \\
\texttt{REAL\_FAKE\_AND\_REAL\_AVG\_DIFF} & 497 & 24.32 & 6.88 & 31.01\% \\
\texttt{REAL\_FAKE\_DIFF} & \textbf{480} & 23.85 & 6.75 & 31.70\% \\
\textbf{\texttt{ALL\_COMBINED} (Default)} & 549 & \textbf{23.29} & \textbf{6.59} & \textbf{32.56\%} \\
\midrule
\textbf{Topic Slice Count $K$ ($K_{\text{eff}}$)} & \textbf{Ranges} & \textbf{2G (ms)} & \textbf{3G (ms)} & \textbf{Sess Util} \\
\midrule
$K = 5$ ($K_{\text{eff}} = 2$) & \textbf{397} & 29.65 & 7.76 & 19.24\% \\
\textbf{$K = 10$ ($K_{\text{eff}} = 3$)} & 460 & 25.57 & \textbf{6.92} & 24.44\% \\
\textbf{$K = 20$ ($K_{\text{eff}} = 5$, Default)} & 537 & \textbf{25.15} & 7.11 & 31.26\% \\
$K = 40$ ($K_{\text{eff}} = 10$) & 589 & 29.75 & 8.84 & 35.79\% \\
$K = 60$ ($K_{\text{eff}} = 15$) & 645 & 34.20 & 10.42 & \textbf{39.39\%} \\
\bottomrule
\end{tabular}%
}
\end{table}

\paragraph{Topic Matrix Normalization, DNNTM Loss, and Slice Count $K$.}
Table~\ref{tab:ablations} (Top) compares the five normalization operators on \texttt{Noto Sans LGC} ($K=20$): Sinkhorn-Knopp bi-stochastic normalization (\texttt{sinkhorn}) achieves the lowest greedy 3G latency (\textbf{6.98\,ms}), 2G latency (\textbf{24.65\,ms}), and highest session utilization (\textbf{31.77\%}) by jointly equalizing row capacities and column competition. Table~\ref{tab:ablations} (Middle) compares all five DNNTM objectives: while pure nonnegative sparse autoencoder reconstruction (\texttt{AUTOENCODER}, $7.55$\,ms) outperforms heuristics, combining it with discriminative entropy-difference learning (\texttt{HYBRID\_AUTOENCODER}, $7.06$\,ms) and full three-term entropy-difference optimization (\texttt{ALL\_COMBINED}, \textbf{6.59\,ms}, $32.56\%$ utilization) yields the sharpest co-occurrence clusters. Table~\ref{tab:ablations} (Bottom) illustrates the RTT vs.\ bandwidth trade-off as $K$ varies from $5$ to $60$: session utilization increases monotonically with $K$ ($19.24\%$ to $39.39\%$), while 2G and DSL mean latency are minimized at $K=20$ (3G mean latency is minimized at $K=10$).

\begin{table}[!htbp]
\centering
\caption{Effect of browsing session length $L \in \{1, 3, 10\}$ on mean per-glyph latency (ms) and session utilization ($L=10$) on \texttt{Noto Sans SC} (\texttt{GIN-3g}).}
\label{tab:session_ablation}
\resizebox{0.99\columnwidth}{!}{%
\begin{tabular}{lrrrr}
\toprule
\textbf{Strategy (\texttt{Noto Sans SC})} & \textbf{$L=1$ (ms)} & \textbf{$L=3$ (ms)} & \textbf{$L=10$ (ms)} & \textbf{Sess Util ($L=10$)} \\
\midrule
\texttt{single\_slice} & 418.33 & 152.09 & 63.47 & 5.55\% \\
\texttt{equal\_bin\_102} & 296.91 & 116.97 & 45.26 & 7.63\% \\
\texttt{most\_1000\_500\_equal\_bin\_100} & 130.06 & 73.57 & 29.38 & 7.82\% \\
\textbf{MLFont Online LDA ($102$)} & 81.79 & 36.68 & 15.93 & \textbf{16.90\%} \\
\textbf{MLFont DNNTM ($102$)} & \textbf{78.20} & \textbf{35.89} & \textbf{15.71} & 16.57\% \\
\textbf{MLFont Range-Compact GA ($102$)} & 89.11 & 40.78 & 17.67 & 14.64\% \\
\bottomrule
\end{tabular}%
}
\end{table}

\paragraph{Browsing Session Length $L$ and Multi-Seed Reproducibility.}
Table~\ref{tab:session_ablation} evaluates \texttt{Noto Sans SC} across cold-cache single-page visits ($L=1$), short sessions ($L=3$), and standard sessions ($L=10$). Even at $L=1$, \textbf{MLFont DNNTM} ($78.20$\,ms) and \textbf{Online LDA} ($81.79$\,ms) reduce latency by \textbf{81.3\%} and \textbf{80.4\%} over \texttt{single\_slice} ($418.33$\,ms) and by \textbf{39.9\%} and \textbf{37.1\%} over \texttt{most\_1000\_500\_equal\_bin\_100} ($130.06$\,ms). Across three independent random seeds ($\{42, 43, 44\}$), MLFont exhibits tight stability: on \texttt{Noto Sans SC}, \texttt{GIN-3g} mean latency and range counts are $15.83 \pm 0.08$\,ms ($10{,}163 \pm 34$ ranges) for Online LDA, $15.83 \pm 0.09$\,ms ($9{,}698 \pm 44$ ranges) for DNNTM, and $17.70 \pm 0.11$\,ms ($5{,}932 \pm 3$ ranges) for Range-Compact GA (Table~\ref{tab:multi_seed_ci} in Appendix~\ref{app:benchmarks}).

\subsection{Real-Website Font Loading Simulation and Headless Chrome Verification}
\label{sec:website_sim}
To let frontend developers evaluate how fast MLFont-sliced webfonts load on arbitrary HTML websites and network conditions before deployment, our library provides an interactive side-by-side browser simulator (\texttt{web/demo.html}), a programmatic HTML/URL website simulator (\texttt{WebsiteFontSimulator} / \texttt{mlfont-simulate}), and an automated Headless Chrome benchmark (\texttt{run\_headless\_chrome\_benchmark}). Across $20$ representative multi-domain websites spanning all five Google Noto fonts (Figure~\ref{fig:website_simulation} and Table~\ref{tab:website_simulation} in Appendix~\ref{app:benchmarks}), MLFont Range-Compact GA achieves \textbf{1.56$\times$ to 6.09$\times$ cold-cache page font load speedups} over \texttt{single\_slice} ($3.43\times$ mean and up to $4.97\times$ on \texttt{Noto Sans SC} with \textbf{72.0\% bandwidth savings}; $6.09\times$ mean and $84.9\%$ bandwidth savings on \texttt{Noto Sans KR}; $65.4\%$--$84.9\%$ across CJK and $44.7\%$ on LGC), while progressive Wave-1 delivery of the Tier-1 \texttt{most} slices renders $84.9\%$--$96.9\%$ of visible CJK page tokens, with first text paint at \textbf{748\,ms} on \texttt{Cable-5Mbps} for \texttt{Noto Sans SC} (vs.\ $9{,}869$\,ms for the unsliced $5{,}987$\,KB font; in \texttt{WebsiteFontSimulator}, all two-tier strategies share the exact physical Tier-1 \texttt{.woff2} slice bytes, and their minor First Text Paint differences reflect the Wave-0 gzipped CSS size of uncompacted \texttt{online\_lda}/\texttt{dnntm} vs.\ \texttt{range\_compact}). Executing real Headless Chrome (\texttt{v155.0.8059.79}) against the physical \texttt{.woff2} slices ($3$ pages $\times$ $3$ strategies per font $= 45$ cases) confirms \textbf{45/45 (100.0\%) exact slice-set agreement} between \texttt{WebsiteFontSimulator} and Chromium Blink's \texttt{document.fonts} engine, achieving \textbf{1.54$\times$ to 6.02$\times$ mean end-to-end \texttt{document.fonts.ready} speedups} and reducing first-page local WOFF2 decompression and OpenType instantiation latency by \textbf{1.17$\times$ to 4.15$\times$}.

\paragraph{Limitations and Systems Considerations.}
Three practical systems factors govern real-world \texttt{unicode-range} deployments: (1) \emph{WOFF2 OpenType table replication:} splitting a font into $K_{\text{total}}$ physical \texttt{.woff2} files replicates shared OpenType tables (\texttt{CFF}/\texttt{glyf}, \texttt{hmtx}, \texttt{GSUB}/\texttt{GPOS}), increasing static CDN storage by $+46.6\%$ to $+175.3\%$ (Table~\ref{tab:woff2_byte_comparison}); across the 8-page \texttt{Noto Sans SC} website session, heuristic slicing (\texttt{equal\_bin\_102} fetching $68.1$ cumulative slices and \texttt{most\_1000\_500\_equal\_bin\_100} fetching $71.8$ cumulative slices) incurs \emph{negative} net session byte savings ($-1.6\%$ and $-7.6\%$) relative to the monolithic WOFF2 font, whereas MLFont restricts cumulative fetches to $29.8$--$39.9$ co-occurring slices and achieves $+38.2\%$ to $+52.9\%$ net session byte savings; (2) \emph{Cross-slice OpenType shaping:} \texttt{GSUB}/\texttt{GPOS} ligatures and kerning cannot span distinct \texttt{@font-face} slice files, so complex shaping scripts (e.g., Arabic/Indic) should group shaping clusters within shared slices; and (3) \emph{Serial RTT surrogate vs.\ HTTP/2 multiplexing:} Tier-2/3 training optimizes the additive 1-RTT per-slice surrogate (Eq.~\ref{eq:rtt_lat}) for linear tractability, whereas \texttt{WebsiteFontSimulator} and Headless Chrome evaluate wave-multiplexed parallel fetching.

\section{Conclusion}
\label{sec:conclusion}
We presented \textbf{MLFont}, a machine learning and evolutionary optimization framework for web font subsetting and \texttt{unicode-range} slicing. By combining Tier-1 Zipfian frequency isolation, Tier-2 Online LDA and DNNTM topic modeling with two-tier contiguous-bin Bayes alignment and Sinkhorn-Knopp normalization, and a Tier-3 Memetic Range-Compact GA with power-law curriculum constraints, 1D Potts-Viterbi relaxation, and sparse-tail island bridging, MLFont cuts multi-page per-glyph network latency by $56.7\%$--$75.3\%$ across open-source Google Noto benchmarks and accelerates real-website font loading by $1.56\times$--$6.09\times$ ($65.4\%$--$84.9\%$ CJK and $44.7\%$ LGC bandwidth reduction, with $45/45$ exact Headless Chrome slice-set parity) while compressing CSS rules by up to $41.7\%$ and guaranteeing $100\%$ glyph coverage.

\clearpage
\bibliographystyle{plain}
{\small
\bibliography{references}
}

\clearpage
\appendix
\input{appendix}

\end{document}

%% file: appendix.tex
\section{Open-Source Codebase and Design Decisions}
\label{app:codebase}

To ensure long-term reproducibility and facilitate adoption by web typography engineers and machine learning researchers, we provide a complete open-source implementation of \textbf{MLFont} (\texttt{mlfont})---an open-source reproduction effort of Google's machine learning font slicing technology \cite{wang2017learning,googlefonts2017post}---under the Apache 2.0 License at \url{https://github.com/wq2012/mlfont} and on PyPI at \url{https://pypi.org/project/mlfont/}.

\subsection{Modular Software Architecture}
The \texttt{mlfont} library is structured into modular, dependency-isolated components with zero proprietary dependencies:
\begin{itemize}[leftmargin=*,nosep]
    \item \texttt{mlfont.schemas}: Strongly typed dataclasses (\texttt{SlicingStrategy}, \texttt{EvaluationParams}, \texttt{DistributionStats}, \texttt{StrategyEvaluationResult}, \texttt{DecayConfig}, \texttt{RangeCompactConfig}, \texttt{GeneticHistoryEvent}) with deterministic JSON and CSS \texttt{unicode-range} serialization (\texttt{U+XXXX-YYYY}) and contiguous interval merging.
    \item \texttt{mlfont.data}: OpenType/TrueType \texttt{cmap} codepoint extraction via \texttt{fontTools} (\texttt{extract\_codepoints\_from\_font}), canonical Unicode \texttt{.nam} subset file parsing, and document-codepoint co-occurrence matrix construction (\texttt{DocumentCorpusProcessor}).
    \item \texttt{mlfont.frequency}: Zipfian document-frequency partitioner (\texttt{get\_slices\_based\_on\_frequency}, \texttt{FrequencySolver}) and contiguous equal-bin splitter (\texttt{get\_equally\_sized\_chunks}).
    \item \texttt{mlfont.maths}: Topic-codepoint affinity matrix normalization operators (\texttt{topic\_sum}, \texttt{row\_sum}, \texttt{global\_max}, spectral $L_2$ \texttt{spectral\_norm}, and Sinkhorn-Knopp bi-stochastic \texttt{sinkhorn}), curriculum decay schedules (\texttt{PowerDecay}, \texttt{ExponentialDecay}), and greedy/diff-constrained capacity-balanced slice assignment (\texttt{balance\_topics\_greedy}, \texttt{balance\_topics\_diff\_constrained}).
    \item \texttt{mlfont.olda}: Variational Bayes Online Latent Dirichlet Allocation (\texttt{OnlineLdaLearner} / \texttt{OnlineLDASolver}) with two-tier Hungarian contiguous-bin Bayes posterior alignment (\texttt{align\_and\_smooth\_topic\_matrix}).
    \item \texttt{mlfont.dnntm}: PyTorch implementation of the Deep Neural Network Topic Model (\texttt{DnnTmLearner} / \texttt{DNNTMSolver}), unifying discriminative non-negative Shannon Entropy Difference Losses (\texttt{ALL\_COMBINED}, \texttt{REAL\_FAKE\_DIFF}, \texttt{REAL\_FAKE\_AND\_REAL\_AVERAGE\_DIFF}) with nonnegative sparse autoencoder reconstruction (\texttt{AUTOENCODER}, \texttt{HYBRID\_AUTOENCODER}).
    \item \texttt{mlfont.range\_compact}: Memetic Range-Compact Genetic Algorithm (\texttt{RangeCompactGeneticLearner} / \texttt{RangeCompactGASolver}) with power-law slice-capacity curriculum $\text{max\_diff}(t)$, Lagrangian 1D Potts-Viterbi dynamic programming (\texttt{potts\_viterbi\_smooth\_topics}), sparse-tail monochromatic island bridging (\texttt{build\_island\_bridged\_topics}), tier-capacity invariance, single-point crossover, seven mutation operators, and boundary-preserving batch capacity repair.
    \item \texttt{mlfont.baselines}: Reference implementations of \texttt{single\_slice}, \texttt{equal\_bin}, \texttt{random\_slice}, \texttt{language\_slice}, \texttt{most\_and\_equal\_bin}, and \texttt{greedy\_topic}.
    \item \texttt{mlfont.subsetter}: Parallel physical font subsetting engine (\texttt{slice\_font\_file} / \texttt{slice\_font\_to\_woff2}, \texttt{FontSubsetter}) powered by \texttt{fontTools.subset} and Brotli compression to emit production \texttt{.woff2} slice binaries and \texttt{@font-face} CSS stylesheets.
    \item \texttt{mlfont.evaluator}: Multi-page browsing session simulator (\texttt{FontLoadingEvaluator}) supporting arbitrary session lengths $L$, RTT/bandwidth profiles, and both canonical ($200$\,B/glyph) and exact physical WOFF2 slice byte sizes.
    \item \texttt{mlfont.simulator}: Real-website HTML/URL font-loading simulator (\texttt{WebsiteFontSimulator}, \texttt{mlfont-simulate}) and automated Headless Chrome benchmark (\texttt{run\_headless\_chrome\_benchmark}), paired with an interactive side-by-side browser simulator (\texttt{web/demo.html}).
    \item \texttt{mlfont.export}: Safe, zero-code-execution model serialization using Hugging Face \texttt{safetensors} (\texttt{model.safetensors}) alongside \texttt{.nam} subset files, held-out \texttt{test\_samples.json}, and automated Model Card generation.
\end{itemize}

\subsection{Algorithmic Summary and Open-Source Capabilities}
Table~\ref{tab:arch_alignment} summarizes the core algorithmic components and open-source engineering capabilities of the \texttt{mlfont} toolkit, spanning two-tier Zipfian/topic decomposition, Sinkhorn-Knopp bi-stochastic normalization, unified discriminative + nonnegative autoencoder DNNTM objectives, memetic 1D Potts-Viterbi + sparse-tail island-bridging optimization, and end-to-end website/browser benchmarking.

\begin{table*}[!htbp]
\centering
\caption{Summary of core algorithmic modules and engineering capabilities in the open-source \texttt{mlfont} library.}
\label{tab:arch_alignment}
\resizebox{0.98\textwidth}{!}{%
\begin{tabular}{ll}
\toprule
\textbf{System Component} & \textbf{Open-Source Implementation (\texttt{mlfont})} \\
\midrule
Data \& Corpus Processing & Pure Python 3.10+ with NumPy \& SciPy sparse co-occurrence matrices over public Wikimedia/Wikipedia corpora (\texttt{wikimedia/wikipedia}) \\
Tier-1 Frequency Isolation & Deterministic Zipfian document-frequency partitioner (\texttt{FrequencySolver}, \texttt{most\_1..M}) \\
Online LDA Solver & \texttt{OnlineLdaLearner} with two-tier Hungarian contiguous-bin Bayes alignment + Sinkhorn-Knopp bi-stochastic scaling \\
Neural Topic Model (DNNTM) & PyTorch \texttt{DnnTmLearner} unifying non-negative Shannon Entropy Difference Losses with Nonnegative Sparse Autoencoder reconstruction \\
Combinatorial Optimizer & Memetic \texttt{RangeCompactGeneticLearner} (\texttt{PowerDecay} $p=3.0$ + Lagrangian 1D Potts-Viterbi DP + Sparse-Tail Island Bridging) \\
Physical Font Subsetting & \texttt{fontTools.subset} (\texttt{woff2} Brotli compression) + CSS3 \texttt{@font-face} \texttt{unicode-range} stylesheet generator \\
Website \& Browser Simulator & \texttt{WebsiteFontSimulator} (\texttt{mlfont-simulate}) + interactive \texttt{web/demo.html} + Headless Chrome \texttt{document.fonts} benchmark \\
Checkpoint Serialization & Safe, zero-code-execution Hugging Face \texttt{safetensors} (\texttt{model.safetensors}) + JSON strategy bundles \\
\bottomrule
\end{tabular}%
}
\end{table*}

\section{Open-Dataset Benchmark and Pretrained Models}
\label{app:benchmarks}

\subsection{Pretrained Hugging Face Model and WOFF2 Slice Catalog}
We release five pretrained MLFont model bundles and physically subsetted WOFF2 Google Noto font packages on the Hugging Face Hub in the \textbf{MLFont Collection} at \url{https://huggingface.co/collections/wq2012/mlfont} (and in the GitHub repository's \texttt{pretrained\_models/} directory at \url{https://github.com/wq2012/mlfont}). Each bundle provides \texttt{model.safetensors} (storing \texttt{codepoints}, \texttt{remaining\_codepoints}, \texttt{frequency\_array}, \texttt{lambda\_matrix}, \texttt{dnntm\_lambda\_matrix}, and \texttt{dna\_assignment}), \texttt{strategy.json}, \texttt{olda\_strategy.json}, \texttt{dnntm\_strategy.json}, \texttt{test\_samples.json}, \texttt{font\_sliced.css}, physical \texttt{woff2/slice\_*.woff2} files, and \texttt{nam/slice\_*.nam} files (Table~\ref{tab:pretrained_catalog}).

\begin{table*}[!htbp]
\centering
\caption{Summary of released pretrained MLFont models and physically sliced WOFF2 Google Noto font bundles on Hugging Face (\url{https://huggingface.co/collections/wq2012/mlfont}).}
\label{tab:pretrained_catalog}
\resizebox{0.99\textwidth}{!}{%
\begin{tabular}{lllrrrrr}
\toprule
\textbf{Hugging Face Repository} & \textbf{Font Binary} & \textbf{Languages / Scripts} & \textbf{Codepoints} & \textbf{Tier-1 (\texttt{most})} & \textbf{Topics $K$ ($K_{\text{eff}}$)} & \textbf{Total Slices} & \textbf{Physical WOFF2 (CSS)} \\
\midrule
\href{https://huggingface.co/wq2012/mlfont-notosans-lgc}{\texttt{wq2012/mlfont-notosans-lgc}} & \texttt{NotoSans-Regular.ttf} & en, de, fr, pl, ru, el & $2,297$ & $\emptyset$ & $20$ ($5$) & $20$ & $434.1$\,KB ($7.8$\,KB CSS) \\
\href{https://huggingface.co/wq2012/mlfont-notosanskr}{\texttt{wq2012/mlfont-notosanskr}} & \texttt{NotoSansKR-Regular.otf} & Korean (\texttt{ko}, Hangul/Hanja) & $22,853$ & $(300, 300)$ & $50$ ($6$) & $52$ & $5.29$\,MB ($78.0$\,KB CSS) \\
\href{https://huggingface.co/wq2012/mlfont-notosansjp}{\texttt{wq2012/mlfont-notosansjp}} & \texttt{NotoSansJP-Regular.otf} & Japanese (\texttt{ja}, Kanji/Kana) & $16,663$ & $(500, 500)$ & $60$ ($12$) & $62$ & $4.67$\,MB ($85.3$\,KB CSS) \\
\href{https://huggingface.co/wq2012/mlfont-notosanstc}{\texttt{wq2012/mlfont-notosanstc}} & \texttt{NotoSansTC-Regular.otf} & Trad.\ Chinese (\texttt{zh-Hant}) & $20,637$ & $(600, 600)$ & $80$ ($20$) & $82$ & $5.94$\,MB ($104.1$\,KB CSS) \\
\href{https://huggingface.co/wq2012/mlfont-notosanssc}{\texttt{wq2012/mlfont-notosanssc}} & \texttt{NotoSansSC-Regular.otf} & Simp.\ Chinese (\texttt{zh-Hans}) & $30,759$ & $(1000, 500)$ & $100$ ($18$) & $102$ & $8.89$\,MB ($83.9$\,KB CSS) \\
\bottomrule
\end{tabular}%
}
\end{table*}

\subsection{Complete Multi-Network Benchmark Results (\texttt{GIN-2g}, \texttt{GIN-3g}, \texttt{DSL}, \texttt{ADSL2+})}
Table~\ref{tab:full_network_results} reports the complete evaluation metrics across \texttt{GIN-2g} ($\tau_{\text{RTT}}=1200$\,ms, $\beta_{\text{BW}}=12.5$\,B/ms), \texttt{GIN-3g} ($\tau_{\text{RTT}}=400$\,ms, $\beta_{\text{BW}}=50$\,B/ms), \texttt{residential-DSL} ($\tau_{\text{RTT}}=50$\,ms, $\beta_{\text{BW}}=250$\,B/ms), and \texttt{residential-ADSL2+} ($\tau_{\text{RTT}}=50$\,ms, $\beta_{\text{BW}}=2500$\,B/ms) for all five Google Noto fonts. Table~\ref{tab:woff2_byte_comparison} compares the canonical $200$\,B/glyph network model against exact physical WOFF2 slice byte sizes $B_{\text{WOFF2}}(S_k)$ (which average $193.5$\,B/glyph on LGC and $242.9$--$303.2$\,B/glyph on CJK due to replicated OpenType layout tables across slices), Table~\ref{tab:multi_seed_ci} reports multi-seed stability across random seeds $\{42, 43, 44\}$, and Figure~\ref{fig:website_simulation} and Table~\ref{tab:website_simulation} report the real-website simulation and Headless Chrome (\texttt{v155.0.8059.79}) benchmark across all five Noto fonts.

\makeatletter
\setlength{\@fptop}{0pt}
\setlength{\@fpsep}{14pt}
\setlength{\@fpbot}{0pt plus 1fil}
\setlength{\@dblfptop}{0pt}
\setlength{\@dblfpsep}{14pt}
\setlength{\@dblfpbot}{0pt plus 1fil}
\makeatother

\begin{table*}[t]
\centering
\caption{Complete cross-network latency benchmark results (Mean and P95 per-glyph latency in ms, $L=10$ pages/session) across \textbf{\texttt{GIN-2g}}, \textbf{\texttt{GIN-3g}}, \textbf{\texttt{residential-DSL}}, and \textbf{\texttt{residential-ADSL2+}} for all five Google Noto fonts.}
\label{tab:full_network_results}
\resizebox{0.99\textwidth}{!}{%
\begin{tabular}{llrrrrrrrr}
\toprule
\textbf{Font Family} & \textbf{Strategy} & \textbf{CSS Ranges} & \textbf{2G Mean} & \textbf{2G P95} & \textbf{3G Mean} & \textbf{3G P95} & \textbf{DSL Mean} & \textbf{DSL P95} & \textbf{ADSL2+ Mean} \\
\midrule
\texttt{Noto Sans LGC} & \texttt{single\_slice} & 123 & 60.26 & 566.45 & 15.22 & 143.10 & 3.00 & 28.17 & \textbf{0.37} \\
($2,297$ cps) & \texttt{equal\_bin\_20} & 140 & 33.37 & 153.43 & 9.44 & 43.41 & 1.56 & 7.17 & 0.65 \\
 & \texttt{random\_slice\_20} & 2,180 & 95.82 & 906.75 & 27.11 & 256.54 & 4.48 & 42.35 & 1.87 \\
 & \texttt{language\_slice} ($300$ unass.) & 92 & 40.17 & 258.54 & 10.62 & 66.28 & 1.95 & 12.76 & 0.45 \\
 & \textbf{MLFont Online LDA (20)} & 537 & 25.15 & 144.76 & 7.11 & 40.95 & 1.17 & 6.76 & 0.49 \\
 & \textbf{MLFont DNNTM (20)} & 549 & \textbf{23.29} & \textbf{112.70} & \textbf{6.59} & \textbf{31.88} & \textbf{1.09} & \textbf{5.26} & 0.45 \\
 & \textbf{MLFont Range-Compact GA (20)} & 384 & 25.05 & 142.01 & 7.06 & 40.06 & 1.17 & 6.64 & 0.48 \\
\midrule
\texttt{Noto Sans KR} & \texttt{single\_slice} & 4,598 & 250.67 & 1739.45 & 62.73 & 435.34 & 12.53 & 86.92 & 1.28 \\
($22,853$ cps) & \texttt{equal\_bin\_52} & 4,635 & 259.74 & 1338.41 & 68.09 & 350.87 & 12.67 & 65.29 & 2.69 \\
 & \texttt{random\_slice\_52} & 22,493 & 270.86 & 2029.63 & 71.01 & 532.06 & 13.21 & 99.02 & 2.80 \\
 & \texttt{language\_slice} ($4,322$ unass.) & 4,579 & 203.56 & 1514.81 & 51.20 & 380.39 & 10.15 & 75.57 & 1.16 \\
 & \texttt{most\_300\_300\_equal\_bin\_50} & 5,608 & 197.85 & 715.15 & 51.88 & 187.72 & 9.65 & 34.86 & 2.05 \\
 & \textbf{MLFont Online LDA (52)} & 8,327 & \textbf{66.59} & 348.80 & \textbf{17.48} & 91.73 & \textbf{3.25} & 16.99 & \textbf{0.70} \\
 & \textbf{MLFont DNNTM (52)} & 8,354 & 66.97 & \textbf{331.00} & 17.59 & \textbf{87.21} & 3.26 & \textbf{16.10} & 0.71 \\
 & \textbf{MLFont Range-Compact GA (52)} & 6,769 & 80.22 & 354.26 & 21.06 & 93.00 & 3.91 & 17.27 & 0.84 \\
\midrule
\texttt{Noto Sans JP} & \texttt{single\_slice} & 4,830 & 94.75 & 701.82 & 23.72 & 175.72 & 4.73 & 35.06 & 0.49 \\
($16,663$ cps) & \texttt{equal\_bin\_62} & 4,876 & 95.38 & 698.56 & 25.58 & 187.34 & 4.60 & 33.66 & 1.24 \\
 & \texttt{random\_slice\_62} & 16,450 & 114.88 & 893.64 & 30.81 & 239.66 & 5.53 & 43.06 & 1.49 \\
 & \texttt{language\_slice} ($7,632$ unass.) & 5,617 & 45.20 & 305.76 & 11.42 & 77.00 & 2.25 & 15.22 & 0.28 \\
 & \texttt{most\_500\_500\_equal\_bin\_60} & 5,975 & 76.50 & 387.35 & 20.50 & 103.64 & 3.69 & 18.69 & 0.99 \\
 & \textbf{MLFont Online LDA (62)} & 8,888 & 34.87 & 201.00 & 9.31 & 53.61 & \textbf{1.68} & 9.71 & \textbf{0.44} \\
 & \textbf{MLFont DNNTM (62)} & 8,710 & \textbf{34.80} & \textbf{188.72} & \textbf{9.30} & \textbf{50.47} & \textbf{1.68} & \textbf{9.12} & \textbf{0.44} \\
 & \textbf{MLFont Range-Compact GA (62)} & 7,298 & 42.53 & 223.47 & 11.37 & 59.70 & 2.05 & 10.79 & 0.54 \\
\midrule
\texttt{Noto Sans TC} & \texttt{single\_slice} & 5,036 & 125.58 & 704.12 & 31.43 & 176.24 & 6.28 & 35.18 & \textbf{0.64} \\
($20,637$ cps) & \texttt{equal\_bin\_82} & 5,098 & 115.29 & 689.78 & 31.03 & 185.64 & 5.54 & 33.17 & 1.55 \\
 & \texttt{random\_slice\_82} & 20,436 & 144.57 & 922.42 & 38.91 & 248.25 & 6.95 & 44.36 & 1.94 \\
 & \texttt{language\_slice} ($11,699$ unass.) & 5,767 & 52.25 & 282.95 & 13.23 & 71.40 & 2.60 & 14.08 & 0.33 \\
 & \texttt{most\_600\_600\_equal\_bin\_80} & 6,672 & 92.66 & 446.08 & 24.90 & 119.91 & 4.46 & 21.47 & 1.23 \\
 & \textbf{MLFont Online LDA (82)} & 11,867 & \textbf{50.11} & 250.61 & \textbf{13.43} & 67.58 & \textbf{2.42} & 12.04 & 0.65 \\
 & \textbf{MLFont DNNTM (82)} & 11,410 & 51.65 & \textbf{226.65} & 13.84 & \textbf{60.72} & 2.49 & \textbf{10.93} & 0.67 \\
 & \textbf{MLFont Range-Compact GA (82)} & 8,743 & 56.19 & 265.70 & 15.07 & 71.62 & 2.71 & 12.79 & 0.73 \\
\midrule
\texttt{Noto Sans SC} & \texttt{single\_slice} & 418 & 253.67 & 1629.56 & 63.47 & 407.72 & 12.68 & 81.45 & 1.29 \\
($30,759$ cps) & \texttt{equal\_bin\_102} & 519 & 169.75 & 1165.36 & 45.26 & 310.69 & 8.21 & 56.33 & 2.09 \\
 & \texttt{random\_slice\_102} & 30,457 & 242.74 & 1930.38 & 64.71 & 514.64 & 11.73 & 93.32 & 2.99 \\
 & \texttt{language\_slice} ($19,659$ unass.) & 6,646 & 86.54 & 554.31 & 21.82 & 139.59 & 4.31 & 27.61 & 0.51 \\
 & \texttt{most\_1000\_500\_equal\_bin\_100} & 2,965 & 110.52 & 502.64 & 29.38 & 131.07 & 5.35 & 24.59 & 1.32 \\
 & \textbf{MLFont Online LDA (102)} & 10,177 & 60.16 & 314.64 & 15.93 & 83.23 & 2.92 & 15.27 & 0.69 \\
 & \textbf{MLFont DNNTM (102)} & 9,640 & \textbf{59.30} & \textbf{311.97} & \textbf{15.71} & \textbf{82.49} & \textbf{2.88} & \textbf{15.14} & \textbf{0.68} \\
 & \textbf{MLFont Range-Compact GA (102)} & 5,935 & 66.67 & 356.73 & 17.67 & 94.33 & 3.23 & 17.32 & 0.78 \\
\bottomrule
\end{tabular}%
}
\end{table*}

\begin{table*}[t]
\centering
\caption{Comparison of canonical $200$\,B/glyph simulation vs.\ exact physical WOFF2 slice byte sizes $B_{\text{WOFF2}}(S_k)$ for MLFont Range-Compact GA ($L=10$).}
\label{tab:woff2_byte_comparison}
\resizebox{0.82\textwidth}{!}{%
\begin{tabular}{lrrrrrr}
\toprule
\textbf{Font Family} & \textbf{Slices} & \textbf{Avg WOFF2 B/glyph} & \textbf{\texttt{GIN-3g} ($200$\,B)} & \textbf{\texttt{GIN-3g} (Physical WOFF2)} & \textbf{\texttt{DSL} ($200$\,B)} & \textbf{\texttt{DSL} (Physical WOFF2)} \\
\midrule
\texttt{Noto Sans LGC} & 20 & $193.5$\,B & 7.06 / 40.06 & 6.71 / 39.49 & 1.17 / 6.64 & 1.10 / 6.47 \\
\texttt{Noto Sans KR} & 52 & $242.9$\,B & 21.06 / 93.00 & 25.52 / 116.06 & 3.91 / 17.27 & 4.80 / 22.05 \\
\texttt{Noto Sans JP} & 62 & $293.7$\,B & 11.37 / 59.70 & 16.29 / 87.35 & 2.05 / 10.79 & 3.04 / 16.37 \\
\texttt{Noto Sans TC} & 82 & $302.0$\,B & 15.07 / 71.62 & 20.75 / 99.90 & 2.71 / 12.79 & 3.85 / 18.42 \\
\texttt{Noto Sans SC} & 102 & $303.2$\,B & 17.67 / 94.33 & 23.96 / 126.04 & 3.23 / 17.32 & 4.49 / 23.67 \\
\bottomrule
\end{tabular}%
}

\vspace{1.2em}
\caption{Multi-seed reproducibility across 3 independent random seeds ($\{42, 43, 44\}$): mean $\pm$ standard deviation of \texttt{GIN-3g} per-glyph latency (ms) and total CSS \texttt{unicode-range} interval count.}
\label{tab:multi_seed_ci}
\resizebox{0.76\textwidth}{!}{%
\begin{tabular}{llrr}
\toprule
\textbf{Font Family} & \textbf{MLFont Learner} & \textbf{\texttt{GIN-3g} Mean $\pm$ Std (ms)} & \textbf{CSS Ranges $\pm$ Std} \\
\midrule
\texttt{Noto Sans LGC} ($2{,}297$ cps) & Online LDA ($K=20$) & $7.11 \pm 0.03$ & $528.0 \pm 17.9$ \\
 & DNNTM ($K=20$) & $6.66 \pm 0.20$ & $539.0 \pm 7.3$ \\
 & Range-Compact GA ($K=20$) & $7.19 \pm 0.10$ & $339.0 \pm 42.0$ \\
\midrule
\texttt{Noto Sans SC} ($30{,}759$ cps) & Online LDA ($102$ slices) & $15.83 \pm 0.08$ & $10{,}162.7 \pm 33.8$ \\
 & DNNTM ($102$ slices) & $15.83 \pm 0.09$ & $9{,}698.0 \pm 44.5$ \\
 & Range-Compact GA ($102$ slices) & $17.70 \pm 0.11$ & $5{,}931.7 \pm 3.4$ \\
\bottomrule
\end{tabular}%
}

\vspace{1.2em}
\caption*{\textbf{Algorithm 1:} Memetic Range-Compact Genetic Algorithm (\texttt{RangeCompactGeneticLearner})}
\small
\begin{tabular}{p{0.96\textwidth}}
\toprule
\textbf{Require:} Sorted tail codepoints $\mathcal{C}_{\text{rem}} = (c_1, \dots, c_{V_{\text{rem}}})$, normalized topic matrix $\mathbf{\Phi} \in \mathbb{R}_+^{K \times V_{\text{rem}}}$, survivors $n_s=20$, mutations $M_{\text{mut}}=40$, crossovers $M_{\text{cross}}=20$, iterations $T=1200$, decay bounds $(\Delta_{\min}, \Delta_{\max})$, power $p=3.0$, weight $w_r \in \{0.012, 0.18\}$. \\
\textbf{Ensure:} Best-found range-compact slice assignment $\mathbf{s}^* \in \{0, \dots, K-1\}^{V_{\text{rem}}}$ (a heuristic, not a certified global optimum). \\
\midrule
1: Identify foreground topic tier $\mathcal{T}_{\text{fg}}$ (the $K_{\text{eff}}$ most populated greedy topics when $V_{\text{rem}} > 5{,}000$, and all $K$ topics on dense LGC fonts) vs.\ background archive tier $\mathcal{T}_{\text{bg}} = \{0, \dots, K-1\} \setminus \mathcal{T}_{\text{fg}}$; initialize seeds $\mathbf{s}^{(1)}$ (\texttt{balance\_topics\_diff\_constrained}), $\mathbf{s}^{(2)}$ (\texttt{balance\_topics\_greedy}), and $\mathbf{s}^{(3)}$ (\texttt{init\_gene\_by\_equal\_binning}), repaired via tier-preserving \texttt{\_repair\_max\_diff} to satisfy $\max_k |S_k| - \min_k |S_k| \le \Delta_{\max}$. \\
2: \textbf{for} generation $t = 0$ \textbf{to} $T$ \textbf{do} \\
3: \quad $\text{max\_diff}(t) \leftarrow \lfloor \Delta_{\min} + (\Delta_{\max} - \Delta_{\min}) \max(0, 1 - t/T)^p \rfloor$. \\
4: \quad Sample $M_{\text{mut}}$ offspring via the 7 mutation operators (\texttt{MERGE\_SHORT\_RANGE}, \texttt{LEFT/RIGHT\_SHIFT}, \texttt{PURELY\_RANDOM}, \texttt{SWAP\_TOPIC}, \texttt{INCREASE\_MIN/DECREASE\_MAX\_TOPIC}) preserving $\mathcal{T}_{\text{fg}}/\mathcal{T}_{\text{bg}}$ tier membership, and repair capacity violations via boundary-first \texttt{\_repair\_max\_diff}$(\cdot, \text{max\_diff}(t))$. \\
5: \quad Sample $M_{\text{cross}}$ offspring via single-point crossover and retain valid offspring satisfying $\text{max\_diff}(t)$. \\
6: \quad At curriculum milestones $t/T \in \{0.15, 0.35, 0.55, 0.75, 0.90\}$, inject memetic candidates via short-run smoothing (\texttt{\_smooth\_beneficial\_runs}) and Lagrangian 1D Potts-Viterbi DP (\texttt{potts\_viterbi\_smooth\_topics}); at $t/T \in \{0.55, 0.75, 0.90\}$ (when $V_{\text{rem}} \ge 2K$), also inject sparse-tail island bridging (\texttt{build\_island\_bridged\_topics}). \\
7: \quad Deduplicate by topic-array hash, evaluate $\mathcal{J}(\mathbf{s}) = w_r C_{\text{range}}(\mathbf{s}) + (1 - w_r) C_{\text{topic}}(\mathbf{s})$, and keep top $n_s$ elites. \\
8: \textbf{end for} \\
9: \textbf{return} Best elite individual $\mathbf{s}^* = \arg\min_{\mathbf{s} \in \mathcal{P}_T} \mathcal{J}(\mathbf{s})$. \\
\bottomrule
\end{tabular}
\end{table*}

\begin{table*}[t]
\centering
\includegraphics[width=0.96\textwidth]{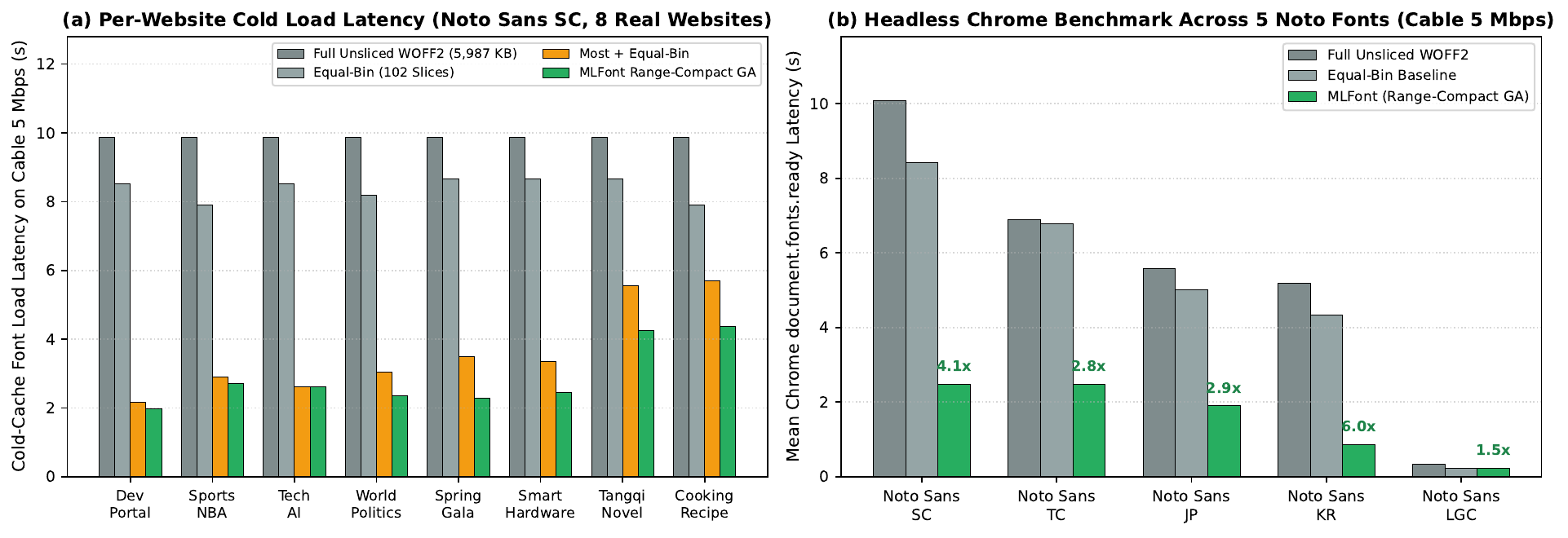}
\captionof{figure}{Real-website font loading simulation and Headless Chrome (\texttt{v155.0.8059.79}) verification across all five Google Noto fonts (\texttt{experiments/website\_simulation\_report.json}): \textbf{(a)} Per-website cold-cache font load latency on \texttt{Noto Sans SC} across $8$ Simplified Chinese websites on \texttt{Cable-5Mbps} ($5$\,Mbps, $28$\,ms RTT) and \textbf{(b)} Real Headless Chrome end-to-end \texttt{document.fonts.ready} latency across all $5$ Noto fonts (mean over all $3$ Chrome-benchmarked pages per font, with $45/45$ exact slice-set parity) showing $1.5\times$ to $6.0\times$ speedup over \texttt{single\_slice} (ratio of means: SC $4.06\times$, TC $2.77\times$, JP $2.94\times$, KR $6.02\times$, LGC $1.54\times$; mean per-page speedup: SC $4.14\times$, TC $2.86\times$, JP $3.08\times$, KR $6.36\times$, LGC $1.66\times$).}
\label{fig:website_simulation}

\vspace{1.2em}
\captionof{table}{Real-Website Font Loading Simulation ($20$ websites across $5$ Google Noto fonts) and Real Headless Chrome (\texttt{v155.0.8059.79}, \texttt{document.fonts} API) Verification on \texttt{GIN-3g} ($400$\,kbps, $400$\,ms RTT) and \texttt{Cable-5Mbps} ($5$\,Mbps, $28$\,ms RTT). Simulator columns are means over the listed pages; \emph{Chrome Dec.} is measured on the first benchmark page of each font. Chrome matched the simulator's exact slice set in all $45$ of $45$ cases ($100.0\%$).}
\label{tab:website_simulation}
{\setlength{\tabcolsep}{4pt}%
\resizebox{0.99\textwidth}{!}{%
\begin{tabular}{llrrrrrrrrr}
\toprule
\textbf{Font (Full WOFF2)} & \textbf{Strategy} & \textbf{Slices} & \textbf{Transfer} & \textbf{Saved} & \textbf{3G Cold} & \textbf{Cable Cold} & \textbf{1st Paint} & \textbf{Warm} & \textbf{Chrome Dec.} & \textbf{Speedup (3G/Cab)} \\
\midrule
\texttt{Noto Sans SC} & \texttt{single\_slice} (Full) & $1.00 / 1$ & $5,988.7$\,KB & $0.0\%$ & $123,450$\,ms & $9,869$\,ms & $9,869$\,ms & $0$\,ms & $230.6$\,ms ($1.00\times$) & $1.00\times$ / $1.00\times$ \\
($5,987.0$\,KB, $8$ p.) & \texttt{equal\_bin\_102} & $55.12 / 102$ & $4,924.0$\,KB & $17.8\%$ & $105,099$\,ms & $8,371$\,ms & $946$\,ms & $295$\,ms & $118.6$\,ms ($1.94\times$) & $1.17\times$ / $1.18\times$ \\
 & \texttt{most\_1000\_500\_eq\_100} & $21.75 / 102$ & $2,111.4$\,KB & $64.7\%$ & $45,201$\,ms & $3,606$\,ms & $721$\,ms & $1,066$\,ms & --- & $2.73\times$ / $2.74\times$ \\
 & \textbf{MLFont Range-Compact ($102$)} & $\mathbf{16.88 / 102}$ & $\mathbf{1,673.7\,\text{KB}}$ & $\mathbf{72.0\%}$ & $\mathbf{36,044\,\text{ms}}$ & $\mathbf{2,880\,\text{ms}}$ & $\mathbf{748\,\text{ms}}$ & $\mathbf{494\,\text{ms}}$ & $\mathbf{55.6\,\text{ms}}$ ($\mathbf{4.15\times}$) & $\mathbf{3.43\times}$ / $\mathbf{3.43\times}$ \\
\midrule
\texttt{Noto Sans TC} & \texttt{single\_slice} (Full) & $1.00 / 1$ & $4,067.0$\,KB & $0.0\%$ & $84,100$\,ms & $6,728$\,ms & $6,728$\,ms & $0$\,ms & $170.5$\,ms ($1.00\times$) & $1.00\times$ / $1.00\times$ \\
($4,047.4$\,KB, $3$ p.) & \texttt{equal\_bin\_82} & $52.00 / 82$ & $3,881.7$\,KB & $4.6\%$ & $83,644$\,ms & $6,662$\,ms & $834$\,ms & $589$\,ms & $182.8$\,ms ($0.93\times$) & $1.01\times$ / $1.01\times$ \\
 & \textbf{MLFont Range-Compact ($82$)} & $\mathbf{16.00 / 82}$ & $\mathbf{1,406.8\,\text{KB}}$ & $\mathbf{65.4\%}$ & $\mathbf{30,431\,\text{ms}}$ & $\mathbf{2,436\,\text{ms}}$ & $\mathbf{665\,\text{ms}}$ & $\mathbf{737\,\text{ms}}$ & $\mathbf{63.8\,\text{ms}}$ ($\mathbf{2.67\times}$) & $\mathbf{2.76\times}$ / $\mathbf{2.76\times}$ \\
\midrule
\texttt{Noto Sans JP} & \texttt{single\_slice} (Full) & $1.00 / 1$ & $3,279.1$\,KB & $0.0\%$ & $67,965$\,ms & $5,437$\,ms & $5,437$\,ms & $0$\,ms & $151.6$\,ms ($1.00\times$) & $1.00\times$ / $1.00\times$ \\
($3,260.2$\,KB, $3$ p.) & \texttt{equal\_bin\_62} & $37.00 / 62$ & $2,873.3$\,KB & $12.4\%$ & $61,789$\,ms & $4,924$\,ms & $859$\,ms & $989$\,ms & $92.8$\,ms ($1.63\times$) & $1.10\times$ / $1.10\times$ \\
 & \textbf{MLFont Range-Compact ($62$)} & $\mathbf{10.33 / 62}$ & $\mathbf{1,062.9\,\text{KB}}$ & $\mathbf{67.6\%}$ & $\mathbf{23,117\,\text{ms}}$ & $\mathbf{1,851\,\text{ms}}$ & $\mathbf{689\,\text{ms}}$ & $\mathbf{568\,\text{ms}}$ & $\mathbf{60.0\,\text{ms}}$ ($\mathbf{2.53\times}$) & $\mathbf{2.94\times}$ / $\mathbf{2.94\times}$ \\
\midrule
\texttt{Noto Sans KR} & \texttt{single\_slice} (Full) & $1.00 / 1$ & $3,048.7$\,KB & $0.0\%$ & $63,246$\,ms & $5,059$\,ms & $5,059$\,ms & $0$\,ms & $147.4$\,ms ($1.00\times$) & $1.00\times$ / $1.00\times$ \\
($3,032.4$\,KB, $3$ p.) & \texttt{equal\_bin\_52} & $24.00 / 52$ & $2,516.7$\,KB & $17.4\%$ & $53,552$\,ms & $4,274$\,ms & $1,120$\,ms & $99$\,ms & $74.4$\,ms ($1.98\times$) & $1.18\times$ / $1.18\times$ \\
 & \textbf{MLFont Range-Compact ($52$)} & $\mathbf{5.67 / 52}$ & $\mathbf{461.2\,\text{KB}}$ & $\mathbf{84.9\%}$ & $\mathbf{10,394\,\text{ms}}$ & $\mathbf{836\,\text{ms}}$ & $\mathbf{287\,\text{ms}}$ & $\mathbf{97\,\text{ms}}$ & $\mathbf{43.5\,\text{ms}}$ ($\mathbf{3.39\times}$) & $\mathbf{6.09\times}$ / $\mathbf{6.06\times}$ \\
\midrule
\texttt{Noto Sans LGC} & \texttt{single\_slice} (Full) & $1.00 / 1$ & $158.4$\,KB & $0.0\%$ & $4,043$\,ms & $316$\,ms & $316$\,ms & $0$\,ms & $28.6$\,ms ($1.00\times$) & $1.00\times$ / $1.00\times$ \\
($157.7$\,KB, $3$ p.) & \texttt{equal\_bin\_20} & $4.00 / 20$ & $87.8$\,KB & $44.6\%$ & $2,599$\,ms & $201$\,ms & $201$\,ms & $32$\,ms & $27.2$\,ms ($1.05\times$) & $1.56\times$ / $1.57\times$ \\
 & \textbf{MLFont Range-Compact ($20$)} & $\mathbf{4.00 / 20}$ & $\mathbf{87.6\,\text{KB}}$ & $\mathbf{44.7\%}$ & $\mathbf{2,596\,\text{ms}}$ & $\mathbf{201\,\text{ms}}$ & $\mathbf{201\,\text{ms}}$ & $\mathbf{0\,\text{ms}}$ & $\mathbf{24.5\,\text{ms}}$ ($\mathbf{1.17\times}$) & $\mathbf{1.56\times}$ / $\mathbf{1.57\times}$ \\
\bottomrule
\end{tabular}%
}}
\end{table*}

\section{Mathematical Proofs and Algorithmic Details}
\label{app:math}

Algorithm~1 summarizes the complete memetic evolutionary loop of \texttt{RangeCompactGeneticLearner}, Proposition~1 establishes the monotonic interval-reduction property motivating its neighborhood-smoothing mutation operators, and Proposition~2 establishes the exact $\mathcal{O}(V_{\text{rem}} K)$ dynamic programming complexity of the Lagrangian 1D Potts-Viterbi relaxation and the strict interval reduction of sparse-tail island bridging.

\begin{proposition}[Interval Monotonicity of Neighborhood Smoothing]
Let $\mathbf{s} = (s_1, \dots, s_{V_{\text{rem}}}) \in \{0, \dots, K-1\}^{V_{\text{rem}}}$ be a slice assignment vector over sorted codepoints $c^{\text{rem}}_1 < \dots < c^{\text{rem}}_{V_{\text{rem}}}$. Suppose index $v \in \{2, \dots, V_{\text{rem}}-1\}$ satisfies $c^{\text{rem}}_v = c^{\text{rem}}_{v-1} + 1$ and $s_v \neq s_{v-1}$. Define the smoothed assignment $\mathbf{s}'$ identical to $\mathbf{s}$ except $s'_v = s_{v-1}$. Then both the per-slice CSS range sum $R(\mathbf{s}) = \sum_{k=0}^{K-1} R(S_k)$ and the global break count $C_{\text{range}}(\mathbf{s})$ satisfy:
\begin{equation}
    R(\mathbf{s}') - R(\mathbf{s}) \le 0, \qquad C_{\text{range}}(\mathbf{s}') - C_{\text{range}}(\mathbf{s}) \le 0,
\end{equation}
and $C_{\text{range}}(\mathbf{s}') - C_{\text{range}}(\mathbf{s}) \in \{0, -1, -2\}$: it equals $-2$ when $c^{\text{rem}}_{v+1} = c^{\text{rem}}_v + 1$ and $s_{v+1} = s_{v-1}$, $0$ when $c^{\text{rem}}_{v+1} = c^{\text{rem}}_v + 1$ and $s_{v+1} = s_v$ (the removed break reappears between $v$ and $v+1$), and $-1$ otherwise.
\end{proposition}
\begin{proof}
Consider the two affected slices $k_1 = s_{v-1}$ and $k_2 = s_v$. Because $c^{\text{rem}}_{v-1} = c^{\text{rem}}_v - 1 \in S_{k_1}$ and $S_{k_1} \cap S_{k_2} = \emptyset$, we have $c^{\text{rem}}_v - 1 \notin S_{k_2}$. Thus $c^{\text{rem}}_v$ is the left endpoint of an interval in $S_{k_2}$, so removing $c^{\text{rem}}_v$ from $S_{k_2}$ yields $\Delta R(S_{k_2}) = 0$ if $c^{\text{rem}}_v + 1 \in S_{k_2}$ and $\Delta R(S_{k_2}) = -1$ otherwise. Inserting $c^{\text{rem}}_v$ into $S_{k_1}$ extends the interval ending at $c^{\text{rem}}_{v-1} = c^{\text{rem}}_v - 1$, yielding $\Delta R(S_{k_1}) \le 0$ (and $\Delta R(S_{k_1}) = -1$ if $c^{\text{rem}}_v + 1 \in S_{k_1}$). For $C_{\text{range}}(\mathbf{s})$, setting $s'_v = s_{v-1}$ eliminates the break between $v-1$ and $v$ ($-1$); the indicator between $v$ and $v+1$ changes from $\mathbb{I}[\text{gap} \lor s_{v+1} \neq s_v]$ to $\mathbb{I}[\text{gap} \lor s_{v+1} \neq s_{v-1}]$, which is $+1$ only in the contiguous case $s_{v+1} = s_v$, $-1$ only in the contiguous case $s_{v+1} = s_{v-1}$, and $0$ otherwise. Hence the net change is never positive; the smoothing mutation operators of Section~\ref{sec:ga} apply it only where it strictly reduces $\mathcal{J}$.
\end{proof}

\begin{proposition}[Linear-in-$K$ Potts-Viterbi Optimality and Island Bridging]
For any dual slice-capacity multiplier vector $\boldsymbol{\mu} \in \mathbb{R}^K$, minimizing the Lagrangian relaxation $\mathcal{L}(\mathbf{s}, \boldsymbol{\mu}) = w_r C_{\text{range}}(\mathbf{s}) + \sum_{v=1}^{V_{\text{rem}}} (\mu_{s_v} - (1-w_r)\Phi_{s_v,v})$ over $\mathbf{s} \in \{0, \dots, K-1\}^{V_{\text{rem}}}$ decomposes into independent 1D Potts chains along maximal contiguous Unicode segments and is solved globally in $\mathcal{O}(V_{\text{rem}} K)$ time. Moreover, absorbing any contiguous background island $[l, r]$ enclosed by identical foreground assignments $s_{l-1} = s_{r+1} = k_{\text{fg}}$ within a contiguous Unicode segment strictly reduces $C_{\text{range}}(\mathbf{s})$ by at least $2$.
\end{proposition}
\begin{proof}
At any non-contiguous boundary $c^{\text{rem}}_{v+1} > c^{\text{rem}}_v + 1$, the indicator $\mathbb{I}[c^{\text{rem}}_{v+1} \neq c^{\text{rem}}_v + 1 \lor s_{v+1} \neq s_v] = 1$ is independent of $(s_v, s_{v+1})$, decoupling the chain across non-contiguous gaps. Within any contiguous segment, let $D_{v,k}$ denote the minimum prefix cost ending in state $s_v = k$. Because the pairwise Potts penalty $w_r \mathbb{I}[k' \neq k]$ takes only two values ($0$ for $k'=k$ and $w_r$ for $k' \neq k$), the Bellman recursion simplifies to $D_{v,k} = \mu_k - (1-w_r)\Phi_{k,v} + \min(D_{v-1,k},\, m_{v-1} + w_r)$ where $m_{v-1} = \min_{k'} D_{v-1,k'}$ is computed once per step $v$ in $\mathcal{O}(K)$ time, yielding $\mathcal{O}(V_{\text{rem}} K)$ total complexity. Finally, if $[l, r]$ is a background run with $s_{l-1} = s_{r+1} = k_{\text{fg}} \notin \{s_l, \dots, s_r\}$ inside a contiguous Unicode block, setting $s'_u = k_{\text{fg}}$ for $u \in [l, r]$ eliminates the two boundary transitions at $(l-1, l)$ and $(r, r+1)$ and any internal transitions within $[l, r]$, giving $C_{\text{range}}(\mathbf{s}') - C_{\text{range}}(\mathbf{s}) \le -2$.
\end{proof}